\documentclass[english,10pt,twocolumn,letterpaper]{article}
\usepackage[T1]{fontenc}
\usepackage[latin9]{inputenc}
\usepackage{array}
\usepackage{url}
\usepackage{multirow}
\usepackage{amsmath}
\usepackage{amsthm}
\usepackage{amssymb}
\usepackage{graphicx}
\usepackage{tablefootnote}

\makeatletter

\let\SF@@footnote\footnote
\def\footnote{\ifx\protect\@typeset@protect
    \expandafter\SF@@footnote
  \else
    \expandafter\SF@gobble@opt
  \fi
}
\expandafter\def\csname SF@gobble@opt \endcsname{\@ifnextchar[%]
  \SF@gobble@twobracket
  \@gobble
}
\edef\SF@gobble@opt{\noexpand\protect
  \expandafter\noexpand\csname SF@gobble@opt \endcsname}
\def\SF@gobble@twobracket[#1]#2{}
\providecommand{\tabularnewline}{\\}
\newcommand{\lyxdot}{.}

\theoremstyle{plain}
\newtheorem{thm}{\protect\theoremname}
\theoremstyle{plain}
\newtheorem{prop}[thm]{\protect\propositionname}

\usepackage{iccv}
\usepackage{times}
\usepackage{epsfig}
\usepackage{graphicx}
\usepackage{amsmath}
\usepackage{amssymb}

\usepackage[pagebackref=true,breaklinks=true,letterpaper=true,colorlinks,bookmarks=false]{hyperref}

\iccvfinalcopy % *** Uncomment this line for the final submission

\def\iccvPaperID{9567} % *** Enter the ICCV Paper ID here

\ificcvfinal\fi

\@ifundefined{showcaptionsetup}{}{%
 \PassOptionsToPackage{caption=false}{subfig}}
\usepackage{subfig}
\makeatother

\usepackage{babel}
\providecommand{\propositionname}{Proposition}
\providecommand{\theoremname}{Theorem}

\begin{document}
\title{Clustering by Maximizing Mutual Information Across Views}
\author{Kien Do, Truyen Tran, Svetha Venkatesh\\
Applied Artificial Intelligence Institute (A2I2), Deakin University,
Geelong, Australia\\
\emph{\{k.do, truyen.tran, svetha.venkatesh\}@deakin.edu.au}}

\maketitle
\ificcvfinal\thispagestyle{empty}\fi

\global\long\def\Expect{\mathbb{E}}%
\global\long\def\Real{\mathbb{R}}%
\global\long\def\Data{\mathcal{D}}%
\global\long\def\Loss{\mathcal{L}}%
\global\long\def\Normal{\mathcal{N}}%
\global\long\def\ELBO{\text{ELBO}}%
\global\long\def\argmin#1{\underset{#1}{\text{argmin }}}%
\global\long\def\argmax#1{\underset{#1}{\text{argmax }}}%
\global\long\def\Model{\text{CRLC}}%

\begin{abstract}
We propose a novel framework for image clustering that incorporates
joint representation learning and clustering. Our method consists
of two heads that share the same backbone network - a ``representation
learning'' head and a ``clustering'' head. The ``representation
learning'' head captures fine-grained patterns of objects at the
instance level which serve as clues for the ``clustering'' head
to extract coarse-grain information that separates objects into clusters.
The whole model is trained in an end-to-end manner by minimizing the
weighted sum of two sample-oriented contrastive losses applied to
the outputs of the two heads. To ensure that the contrastive loss
corresponding to the ``clustering'' head is optimal, we introduce
a novel critic function called ``log-of-dot-product''. Extensive
experimental results demonstrate that our method significantly outperforms
state-of-the-art single-stage clustering methods across a variety
of image datasets, improving over the best baseline by about 5-7\%
in accuracy on CIFAR10/20, STL10, and ImageNet-Dogs. Further, the
``two-stage'' variant of our method also achieves better results
than baselines on three challenging ImageNet subsets.

\end{abstract}

\section{Introduction}

The explosion of unlabeled data, especially visual content in recent
years has led to the growing demand for effective organization of
these data into semantically distinct groups in an unsupervised manner.
Such data clustering facilitates downstream machine learning and reasoning
tasks. Since labels are unavailable, clustering algorithms are mainly
based on the similarity between samples to predict the cluster assignment.
However, common similarity metrics such as cosine similarity or (negative)
Euclidean distance are ineffective when applied to high-dimensional
data like images. Modern image clustering methods \cite{chang2017deep,huang2020deep,ji2019invariant,xie2016unsupervised,yang2017towards,yang2016joint},
therefore, leverage deep neural networks (e.g., CNNs, RNNs) to transform
high-dimensional data into low-dimensional representation vectors
in the latent space and perform clustering in that space. Ideally,
a good clustering model assigns data to clusters to keep inter-group
similarity low while maintaining high intra-group similarity. Most
existing deep clustering methods do not satisfy both of these properties.
For example, autoencoder-based clustering methods \cite{jiang2017variational,yang2017towards,yang2019deepcluster}
often learn representations that capture too much information including
distracting information like background or texture. This prevents
them from computing proper similarity scores between samples at the
cluster-level. Autoencoder-based methods have only been tested on
simple image datasets like MNIST. Another class of methods \cite{chang2017deep,huang2020deep,ji2019invariant}
directly use cluster-assignment probabilities rather than representation
vectors to compute the similarity between samples. These methods can
only differentiate objects belonging to different clusters but not
in the same cluster, hence, may incorrectly group distinct objects
into the same cluster. This leads to low intra-group similarity.

To address the limitations of existing methods, we propose a novel
framework for image clustering called Contrastive Representation Learning
and Clustering (CRLC). CRLC consists of two heads sharing the same
backbone network: a ``representation learning'' head (RL-head) that
outputs a continuous feature vector, and a ``clustering'' head (C-head)
that outputs a cluster-assignment probability vector. The RL-head
computes the similarity between objects at the instance level while
the C-head separates objects into different clusters. The backbone
network serves as a medium for information transfer between the two
heads, allowing the C-head to leverage disciminative fine-grained
patterns captured by the RL-head to extract correct coarse-grained
cluster-level patterns. Via the two heads, CRLC can effectively modulate
the inter-cluster and intra-cluster similarities between samples.
CRLC is trained in an end-to-end manner by minimizing a weighted sum
of two sample-oriented contrastive losses w.r.t. the two heads. To
ensure that the contrastive loss corresponding to the C-head leads
to the tightest InfoNCE lower bound \cite{poole2019variational},
we propose a novel critic called \textquotedblleft log-of-dot-product\textquotedblright{}
to be used in place of the conventional ``dot-product'' critic.

In our experiments, we show that CRLC significantly outperforms a
wide range of state-of-the-art single-stage clustering methods on
five standard image clustering datasets including CIFAR10/20, STL10,
ImageNet10/Dogs. The ``two-stage'' variant of CRLC also achieves
better results than SCAN - a powerful two-stage clustering method
on three challenging ImageNet subsets with 50, 100, and 200 classes.
When some labeled data are provided, CRLC, with only a small change
in its objective, can surpass many state-of-the-art semi-supervised
learning algorithms by a large margin.

In summary, our main contributions are:
\begin{enumerate}
\item A novel framework for joint representation learning and clustering
trained via two sample-oriented contrastive losses on feature and
probability vectors;
\item An optimal critic for the contrastive loss on probability vectors;
and,
\item Extensive experiments and ablation studies to validate our proposed
method against baselines.
\end{enumerate}

\section{Preliminaries}

\subsection{Representation learning by maximizing mutual information across different
views\protect\footnote{Here, we use ``views'' as a generic term to indicate different transformations
of the same data sample.}\label{subsec:ReprViewInfoMax}}

Maximizing mutual information across different views (or ViewInfoMax
for short) allows us to learn view-invariant representations that
capture the semantic information of data important for downstream
tasks (e.g., classification). This learning strategy is also the key
factor behind recent successes in representation learning \cite{hjelm2018learning,oord2018representation,tian2019contrastive,wu2018unsupervised}. 

Since direct computation of mutual information is difficult \cite{mcallester2020formal,song2019understanding},
people usually maximize the variational lower bounds of mutual information
instead. The most common lower bound is InfoNCE \cite{poole2019variational}
whose formula is given by:
\begin{align}
I(X,\tilde{X}) & \geq I_{\text{InfoNCE}}(X,\tilde{X})\label{eq:InfoNCE_1}\\
 & \triangleq\Expect_{p(x_{1:M})p(\tilde{x}|x_{1})}\left[\log\frac{e^{f(\tilde{x},x_{1})}}{\sum_{i=1}^{M}e^{f(\tilde{x},x_{i})}}\right]+\log M\label{eq:InfoNCE_2}\\
 & =-\Loss_{\text{contrast}}+\log M\label{eq:InfoNCE_3}
\end{align}
where $X$, $\tilde{X}$ denote random variables from 2 different
views. $x_{1:M}$ are $M$ samples from $p_{X}$, $\tilde{x}$ is
a sample from $p_{\tilde{X}}$ associated with $x_{1}$. $(\tilde{x},x_{1})$
is called a ``positive'' pair and $(\tilde{x},x_{i})$ ($i=2,...,M$)
are called ``negative'' pairs. $f(x,y)$ is a real value function
called ``critic'' that characterizes the similarity between $x$
and $y$. $\Loss_{\text{contrast}}$ is often known as the ``contrastive
loss'' in other works \cite{chen2020simple,tian2019contrastive}.

Since $\log\frac{e^{f(\tilde{x},x_{1})}}{\sum_{i=1}^{M}e^{f(\tilde{x},x_{i})}}\leq0$,
$I_{\text{InfoNCE}}(X,\tilde{X})$ is upper-bounded by $\log M$.
It means that: i) the InfoNCE bound is very loose if $I(X,\tilde{X})\gg\log M$,
and ii) by increasing $M$, we can achieve a better bound. Despite
being biased, $I_{\text{InfoNCE}}(X,\tilde{X})$ has much lower variance
than other unbiased lower bounds of $I(X,\tilde{X})$ \cite{poole2019variational},
which allows stable training of models.

\paragraph{Implementing the critic}

In practice, $f(\tilde{x},x_{i})$ is implemented as the scaled cosine
similarity between the representations of $\tilde{x}$ and $x_{i}$
as follows:
\begin{align}
f(\tilde{x},x_{i}) & =f(\tilde{z},z_{i})=\tilde{z}^{\top}z_{i}/\tau\label{eq:critic_continuous}
\end{align}
where $\tilde{z}$ and $z_{i}$ are \emph{unit-normed} representation
vectors of $\tilde{x}$ and $x_{i}$, respectively; $\left\Vert \tilde{z}\right\Vert _{2}=\left\Vert z_{i}\right\Vert _{2}=1$.
$\tau>0$ is the ``temperature'' hyperparameter. Interestingly,
$f$ in Eq.~\ref{eq:critic_continuous} matches the theoretically
optimal critic that leads to the tightest InfoNCE bound for unit-normed
representation vectors (detailed explanation in Appdx.~A.4)

In Eq.~\ref{eq:critic_continuous}, we use $f(\tilde{z},z_{i})$
instead of $f(\tilde{x},x_{i})$ to emphasize that the critic $f$
in this context is a function of representations. In regard to this,
we rewrite the contrastive loss in Eq.~\ref{eq:InfoNCE_3} as follows:
\begin{align}
\Loss_{\text{FC}} & =\Expect_{p(x_{1:M})p(\tilde{x}|x_{1})}\left[-\log\frac{e^{f(\tilde{z},z_{1})}}{\sum_{i=1}^{M}e^{f(\tilde{z},z_{i})}}\right]\label{eq:contrast_continuous_1}\\
 & =\Expect_{p(x_{1:M})p(\tilde{x}|x_{1})}\left[\tilde{z}^{\top}z_{1}/\tau-\log\sum_{i=1}^{M}\exp(\tilde{z}^{\top}z_{i}/\tau)\right]\label{eq:contrast_continuous_2}
\end{align}
where FC stands for ``feature contrastive''.

\section{Method}

\subsection{Clustering by maximizing mutual information across different views}

\begin{figure*}
\begin{centering}
\includegraphics[width=0.87\textwidth]{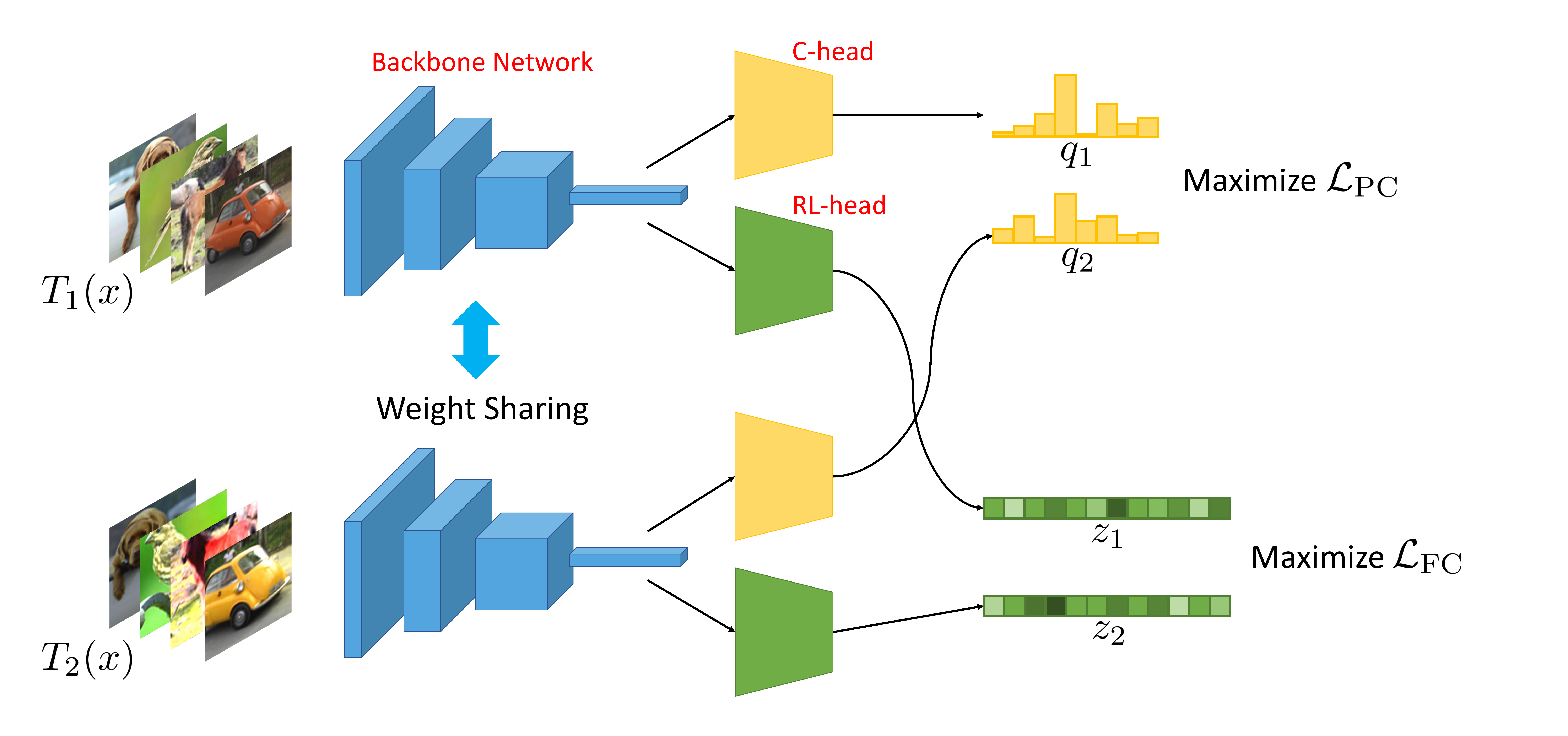}
\par\end{centering}
\caption{Overview of our proposed framework for Contrastive Representation
Learning and Clustering (CRLC). Our framework consists of a ``clustering''
head and a ``representation learning'' head sharing the same backbone
network. $x$ denotes an input images and $T_{1}(x)$, $T_{2}(x)$
denote two different transformations of $x$.\label{fig:ModelArch}}
\end{figure*}

In the clustering problem, we want to learn a parametric classifier
$s_{\theta}$ that maps each unlabeled sample $x_{i}$ to a cluster-assignment
probability vector $q_{i}=(q_{i,1},...,q_{i,C})$ ($C$ is the number
of clusters) whose component $q_{i,c}$ characterizes how likely $x_{i}$
belongs to the cluster $c$ ($c\in\{1,...,C\}$). Intuitively, we
can consider $q_{i}$ as a representation of $x_{i}$ and use this
vector to capture the cluster-level information in $x_{i}$ by leveraging
the ``ViewInfoMax'' idea discussed in Section \ref{subsec:ReprViewInfoMax}.
It leads to the following loss for clustering:

\begin{align}
\Loss_{\text{cluster}} & =\Expect_{p(x_{1:M})p(\tilde{x}|x_{1})}\left[-\log\frac{e^{f(\tilde{q},q_{1})}}{\sum_{i=1}^{M}e^{f(\tilde{q},q_{i})}}\right]-\lambda H(\tilde{Q}_{\text{avg}})\label{eq:contrast_categorical_1}\\
 & =\Loss_{\text{PC}}-\lambda H(\tilde{Q}_{\text{avg}})\label{eq:contrast_categorical_2}
\end{align}
where $\lambda\geq0$ is a coefficient; $\tilde{q}$, $q_{i}$ are
probability vectors associated with $\tilde{x}$ and $x_{i}$, respectively.
$\Loss_{\text{PC}}$ is the \emph{probability contrastive loss} similar
to $\Loss_{\text{FC}}$ (Eq.~\ref{eq:contrast_continuous_1}) but
with feature vectors replaced by probability vectors. $H$ is the
entropy of the \emph{marginal} cluster-assignment probability $\tilde{q}_{\text{avg}}=\Expect_{p(x_{1})p(\tilde{x}|x_{1})}\left[\tilde{q}\right]$.
Here, we maximize $H(\tilde{Q}_{\text{avg}})$ to avoid a degenerate
solution in which all samples fall into the same cluster (e.g., $\tilde{q}$
is one-hot for all samples). However, it is free to use other regularizers
on $\tilde{q}_{\text{avg}}$ rather than $-H(\tilde{Q}_{\text{avg}})$.

\paragraph{Choosing a suitable critic}

It is possible to use the conventional ``dot-product'' critic for
$\Loss_{\text{PC}}$ as for $\Loss_{\text{FC}}$ (Eq.~\ref{eq:critic_continuous}).
However, this will lead to suboptimal results (Section~\ref{par:Comparison-of-different-critics})
since $\Loss_{\text{PC}}$ is applied to categorical probability vectors
rather than continuous feature vectors. Therefore, we need to choose
a suitable critic for $\Loss_{\text{PC}}$ so that the InfoNCE bound
associated with $\Loss_{\text{PC}}$ is tightest. Ideally, $f(\tilde{x},x_{i})$
should match the theoretically optimal critic $f^{*}(\tilde{x},x_{i})$
which is proportional to $\log p(\tilde{x}|x_{i})$ (detailed explanation
in Appdx.~A.3). Denoted by $\tilde{y}$ and $y_{i}$ the cluster
label of $\tilde{x}$ and $x_{i}$ respectively, we then have: 
\begin{align}
\log p(\tilde{x}|x_{i}) & \approx\log\sum_{c=1}^{C}p(\tilde{y}=c|y_{i}=c)\nonumber \\
 & \propto\log\sum_{c=1}^{C}\tilde{q}_{c}q_{i,c}=\log(\tilde{q}^{\top}q_{i})\label{eq:log_dot_prod_critic}
\end{align}
Thus, the most suitable critic is $f(\tilde{q},q_{i})=\log(\tilde{q}^{\top}q_{i})$
which we refer to as the \emph{``log-of-dot-product''} critic. This
critic achieves its maximum value when $\tilde{q}$ and $q_{i}$ are
the same one-hot vectors and its minimum value when $\tilde{q}$ and
$q_{i}$ are different one-hot vectors. Apart from this critic, we
also list other nonoptimal critics in Appdx.~A.1. Empirical comparison
of the ``log-of-dot-product'' critic with other critics is provided
in Section~\ref{par:Comparison-of-different-critics}. 

In addition, to avoid the \emph{gradient saturation} problem of minimizing
$\Loss_{\text{PC}}$ when probabilities are close to one-hot (explanation
in Appdx.~A.5), we smooth out the probabilities as follows:
\[
q=(1-\gamma)q+\gamma r
\]
where $r=\left(\frac{1}{C},...,\frac{1}{C}\right)$ is the uniform
probability vector over $C$ classes; $0\leq\gamma\leq1$ is the smoothing
coefficient set to 0.01 if not otherwise specified.

\paragraph{Implementing the contrastive probability loss\label{par:Implementing-LossPC}}

To implement $\Loss_{\text{PC}}$, we can use either the SimCLR framework
\cite{chen2020simple} or the MemoryBank framework \cite{wu2018unsupervised}.
If the SimCLR framework is chosen, both $\tilde{q}$ and $q_{i}$
($i\in\{1,...,M\}$) are computed directly from $\tilde{x}$ and $x_{i}$
respectively via the parametric classifier $s_{\theta}$. On the other
hand, if the MemoryBank framework is chosen, we maintain a nonparametric
memory bank $\mathcal{M}$ - a matrix of size $N\times C$ containing
the cluster-assignment probabilities of all $N$ training samples,
and update its rows once a new probability is computed as follows:
\begin{equation}
q_{n,t+1}=\alpha q_{n,t}+(1-\alpha)\hat{q}_{n}\label{eq:momentum_update_LossPC}
\end{equation}
where $\alpha$ is the momentum, which is set to 0.5 in our work if
not otherwise specified; $q_{n,t}$ is the probability vector of the
training sample $x_{n}$ at step $t$ corresponding to the $n$-th
row of $\mathcal{M}$; $\hat{q}_{n}=s_{\theta}(x_{n})$ is the new
probability vector. Then, except $\tilde{q}$ computed via $s_{\theta}$
as normal, all $q_{i}$ in Eq.~\ref{eq:contrast_categorical_1} are
sampled uniformly from $\mathcal{M}$. At step 0, all the rows of
$\mathcal{M}$ are initialized with the same probability of $\left(\frac{1}{C},...,\frac{1}{C}\right)$.
We also tried implementing $\Loss_{\text{PC}}$ using the MoCo framework
\cite{he2020momentum} but found that it leads to unstable training.
The main reason is that during the early stage of training, the EMA
model in MoCo often produces inconsistent cluster-assignment probabilities
for different views.

\subsection{Incorporating representation learning}

Due to the limited representation capability of categorical probability
vectors, models trained by minimizing the loss $\Loss_{\text{cluster}}$
in Eq.\ref{eq:contrast_categorical_1} are not able to discriminate
objects in the same cluster. Thus, they may capture suboptimal cluster-level
patterns, which leads to unsatisfactory results.

To overcome this problem, we propose to combine clustering with contrastive
representation learning into a unified framework called CRLC\footnote{CRLC stands for Contrastive Representation Learning and Clustering.}.
As illustrated in Fig.~\ref{fig:ModelArch}, CRLC consists of a ``clustering''
head (C-head) and a ``representation learning'' head (RL-head) sharing
the same backbone network. The backbone network is usually a convolutional
neural network which maps an input image $x$ into a hidden vector
$h$. Then, $h$ is fed to the C-head and the RL-head to produce a
cluster-assignment probability vector $q$ and a continuous feature
vector $z$, respectively. We simultaneously apply the clustering
loss $\Loss_{\text{cluster}}$ (Eq.~\ref{eq:contrast_categorical_2})
and the feature contrastive loss $\Loss_{\text{FC}}$ (Eq.~\ref{eq:contrast_continuous_2})
on $q$ and $z$ respectively and train the whole model with the weighted
sum of $\Loss_{\text{cluster}}$ and $\Loss_{\text{FC}}$ as follows:
\begin{align}
\Loss_{\text{\text{CRLC}}} & =\Loss_{\text{cluster}}+\lambda\Loss_{\text{FC}}\nonumber \\
 & =\Loss_{\text{PC}}-\lambda_{1}H(\tilde{Q}_{\text{avg}})+\lambda_{2}\Loss_{\text{FC}}\label{eq:Loss_CRLC}
\end{align}
where $\lambda_{1},\lambda_{2}\geq0$ are coefficients.

\subsection{A simple extension to semi-supervised learning}

Although $\Model$ is originally proposed for unsupervised clustering,
it can be easily extended to semi-supervised learning (SSL). There
are numerous ways to adjust CRLC so that it can incorporate labeled
data during training. However, within the scope of this work, we only
consider a simple approach which is adding a crossentropy loss on
labeled data to $\Loss_{\text{CRLC}}$. The new loss is given by:

\begin{align}
\Loss_{\text{CRLC-semi}} & =\Loss_{\text{CRLC}}+\lambda\Expect_{(x_{l},y_{l})\sim\Data_{l}}\left[-\log p(y_{l}|x_{l})\right]\nonumber \\
 & =\Loss_{\text{PC}}-\lambda_{1}H(\tilde{Q}_{\text{avg}})+\lambda_{2}\Loss_{\text{FC}}+\lambda_{3}\Loss_{\text{xent}}\label{eq:Loss_CRLC_semi}
\end{align}

We call this variant of CRLC ``CRLC-semi''. Despite its simplicity,
we will empirically show that CRLC-semi outperforms many state-of-the-art
SSL methods when only few labeled samples are available. We conjecture
that the clustering objective arranges the data into disjoint clusters,
making classification easier.

\section{Related Work}

There are a large number of clustering and representation learning
methods in literature. However, within the scope of this paper, we
only discuss works in two related topics, namely, contrastive learning
and deep clustering. 

\subsection{Contrastive Learning}

Despite many recent successes in learning representations, the idea
of contrastive learning appeared long time ago. In 2006, Hadsell et.
al. \cite{hadsell2006dimensionality} proposed a max-margin contrastive
loss and linked it to a mechanical spring system. In fact, from a
probabilistic view, contrastive learning arises naturally when working
with energy-based models. For example, in many problems, we want to
maximize $\log p(y|x)=\log\frac{e^{f(y,x))}}{\sum_{y'\in\mathcal{Y}}e^{f(y',x)}}$
where $y$ is the output associated with a context $x$ and $\mathcal{Y}$
is the set of all possible outputs or vocab. This is roughly equivalent
to maximizing $f(y,x)$ and minimizing $f(y',x)$ for all $y'\neq y$
but in a normalized setting. However, in practice, the size of $\mathcal{Y}$
is usually very large, making the computation of $p(y|x)$ expensive.
This problem was addressed in \cite{mnih2012fast,wu2018unsupervised}
by using Noise Contrastive Estimation (NCE) \cite{gutmann2010noise}
to approximate $p(y|x)$. The basic idea of NCE is to transform the
density estimation problem into a binary classification problem: ``Whether
samples are drawn from the data distribution or from a known noise
distribution?''. Based on NCE, Mikolov et. al. \cite{mikolov2013distributed}
and Oord et. al. \cite{oord2018representation} derived a simpler
contrastive loss which later was referred to as the InfoNCE loss \cite{poole2019variational}
and was adopted by many subsequent works \cite{chen2020simple,dosovitskiy2014discriminative,he2020momentum,misra2020self,tian2019contrastive,ye2019unsupervised}
for learning representations.

Recently, there have been several attempts to leverage inter-sample
statistics obtained from clustering to improve representation learning
on a large scale \cite{asano2019self,caron2018deep,zhuang2019local}.
PCL \cite{li2020prototypical} alternates between clustering data
via K-means and contrasting samples based on their views and their
assigned cluster centroids (or prototypes). SwAV \cite{caron2020unsupervised}
does not contrast two sample views directly but uses one view to predict
the code of assigning the other view to a set of learnable prototypes.
InterCLR \cite{xie2020delving} and ODC \cite{zhan2020online} avoid
offline clustering on the entire training dataset after each epoch
by storing a pseudo-label for every sample in the memory bank (along
with the feature vector) and maintaining a set of cluster centroids.
These pseudo-labels and cluster centroids are updated on-the-fly at
each step via mini-batch K-means.

\subsection{Deep Clustering}

Traditional clustering algorithms such as K-means or Gaussian Mixture
Model (GMM) are mainly designed for low-dimensional vector-like data,
hence, do not perform well on high-dimensional structural data like
images. Deep clustering methods address this limitation by leveraging
the representation power of deep neural networks (e.g., CNNs, RNNs)
to effectively transform data into low-dimensional feature vectors
which are then used as inputs for a clustering objective. For example,
DCN \cite{yang2017towards} applies K-means to the latent representations
produced by an auto-encoder. The reconstruction loss and the K-means
clustering loss are minimized simultaneously. DEC \cite{xie2016unsupervised},
by contrast, uses only an encoder rather than a full autoencoder like
DCN to compute latent representations. This encoder and the cluster
centroids are learned together via a clustering loss proposed by the
authors. JULE \cite{yang2016joint} uses a RNN to implement agglomerative
clustering on top of the representations outputted by a CNN and trains
the two networks in an end-to-end manner. VaDE \cite{jiang2017variational}
regards clustering as an inference problem and learns the cluster-assignment
probabilities of data using a variational framework \cite{kingma2013auto}.
Meanwhile, DAC \cite{chang2017deep} treats clustering as a binary
classification problem: ``Whether a pair of samples belong to the
same cluster or not?''. To obtain a pseudo label for a pair, the
cosine similarity between the cluster-assignment probabilities of
the two samples in that pair is compared with an adaptive threshold.
IIC \cite{ji2019invariant} learns cluster assignments via maximizing
the mutual information between clusters under two different data augmentations.
PICA \cite{huang2020deep}, instead, minimizes the contrastive loss
derived from the the mutual information in IIC. While the cluster
contrastive loss in PICA is cluster-oriented and can have at most
$C$ negative pairs ($C$ is the number of clusters). Our probability
contrastive loss, by contrast, is sample-oriented and can have as
many negative pairs as the number of training data. Thus, in theory,
our proposed model can capture more information than PICA. In real
implementation, in order to gain more information from data, PICA
has to make use of the ``over-clustering'' trick \cite{ji2019invariant}.
It alternates between minimizing $\Loss_{\text{PICA}}$ for $C$ clusters
and minimizing $\Loss_{\text{PICA}}$ for $kC$ clusters ($k>1$ denotes
the ``over-clustering'' coefficient). DRC \cite{zhong2020deep}
and CC \cite{li2020contrastive} enhances PICA by combining clustering
with contrastive representation learning, which follows the same paradigm
as our proposed CRLC. However, like PICA, DRC and CC uses cluster-oriented
representations rather than sample-oriented representations.

In addition to end-to-end deep clustering methods, some multi-stage
clustering methods have been proposed recently \cite{park2020improving,van2020scan}.
The most notable one is SCAN \cite{van2020scan}. This method uses
representations learned via contrastive learning during the first
stage to find nearest neighbors for every sample in the training set.
In the second stage, neighboring samples are forced to have similar
cluster-assignment probabilities. Our probability contrastive loss
can easily be extended to handle neighboring samples (see Section~\ref{subsec:Two-stage-training}).

\section{Experiments}

\begin{table*}
\begin{centering}
\begin{tabular}{|c|ccc|ccc|ccc|ccc|ccc|}
\hline 
{\small{}Dataset} & \multicolumn{3}{c|}{{\small{}CIFAR10}} & \multicolumn{3}{c|}{{\small{}CIFAR20}} & \multicolumn{3}{c|}{{\small{}STL10}} & \multicolumn{3}{c|}{{\small{}ImageNet10}} & \multicolumn{3}{c|}{{\small{}ImageNet-Dogs}}\tabularnewline
\hline 
{\small{}Metric} & {\small{}ACC} & {\small{}NMI} & {\small{}ARI} & {\small{}ACC} & {\small{}NMI} & {\small{}ARI} & {\small{}ACC} & {\small{}NMI} & {\small{}ARI} & {\small{}ACC} & {\small{}NMI} & {\small{}ARI} & {\small{}ACC} & {\small{}NMI} & {\small{}ARI}\tabularnewline
\hline 
\hline 
{\small{}JULE \cite{yang2016joint}} & {\small{}27.2} & {\small{}19.2} & {\small{}13.8} & {\small{}13.7} & {\small{}10.3} & {\small{}3.3} & {\small{}27.7} & {\small{}18.2} & {\small{}16.4} & {\small{}30.0} & {\small{}17.5} & {\small{}13.8} & {\small{}13.8} & {\small{}5.4} & {\small{}2.8}\tabularnewline
\hline 
{\small{}DEC \cite{xie2016unsupervised}} & {\small{}30.1} & {\small{}25.7} & {\small{}16.1} & {\small{}18.5} & {\small{}13.6} & {\small{}5.0} & {\small{}35.9} & {\small{}27.6} & {\small{}18.6} & {\small{}38.1} & {\small{}28.2} & {\small{}20.3} & {\small{}19.5} & {\small{}12.2} & {\small{}7.9}\tabularnewline
\hline 
{\small{}DAC \cite{chang2017deep}} & {\small{}52.2} & {\small{}39.6} & {\small{}30.6} & {\small{}23.8} & {\small{}18.5} & {\small{}8.8} & {\small{}47.0} & {\small{}36.6} & {\small{}25.7} & {\small{}52.7} & {\small{}39.4} & {\small{}30.2} & {\small{}27.5} & {\small{}21.9} & {\small{}11.1}\tabularnewline
\hline 
{\small{}DDC \cite{chang2019deep}} & {\small{}52.4} & {\small{}42.4} & {\small{}32.9} & {\small{}-} & {\small{}-} & {\small{}-} & {\small{}48.9} & {\small{}37.1} & {\small{}26.7} & {\small{}57.7} & {\small{}43.3} & {\small{}34.5} & {\small{}-} & {\small{}-} & {\small{}-}\tabularnewline
\hline 
{\small{}DCCM \cite{wu2019deep}} & {\small{}62.3} & {\small{}49.6} & {\small{}40.8} & {\small{}32.7} & {\small{}28.5} & {\small{}17.3} & {\small{}48.2} & {\small{}37.6} & {\small{}26.2} & {\small{}70.1} & {\small{}60.8} & {\small{}55.5} & {\small{}38.3} & {\small{}32.1} & {\small{}18.2}\tabularnewline
\hline 
{\small{}IIC \cite{ji2019invariant}} & {\small{}61.7} & {\small{}-} & {\small{}-} & {\small{}25.7} & {\small{}-} & {\small{}-} & {\small{}61.0} & {\small{}-} & {\small{}-} & {\small{}-} & {\small{}-} & {\small{}-} & {\small{}-} & {\small{}-} & {\small{}-}\tabularnewline
\hline 
{\small{}MCR2 \cite{yu2020learning}} & {\small{}68.4} & {\small{}63.0} & {\small{}50.8} & {\small{}34.7} & {\small{}36.2} & {\small{}16.7} & {\small{}49.1} & {\small{}44.6} & {\small{}29.0} & {\small{}-} & {\small{}-} & {\small{}-} & {\small{}-} & {\small{}-} & {\small{}-}\tabularnewline
\hline 
{\small{}PICA \cite{huang2020deep}} & {\small{}69.6} & {\small{}59.1} & {\small{}51.2} & {\small{}33.7} & {\small{}31.0} & {\small{}17.1} & {\small{}71.3} & {\small{}61.1} & {\small{}53.1} & {\small{}87.0} & {\small{}80.2} & {\small{}76.1} & {\small{}35.2} & {\small{}35.2} & {\small{}20.1}\tabularnewline
\hline 
{\small{}DRC \cite{zhong2020deep}} & {\small{}72.7} & {\small{}62.1} & {\small{}54.7} & {\small{}36.7} & {\small{}35.6} & {\small{}20.8} & {\small{}74.7} & {\small{}64.4} & {\small{}56.9} & \textbf{\small{}88.4} & {\small{}83.0} & \textbf{\small{}79.8} & {\small{}38.9} & {\small{}38.4} & {\small{}23.3}\tabularnewline
\hline 
\hline 
{\small{}C-head only} & {\small{}66.9} & {\small{}56.9} & {\small{}47.5} & {\small{}37.7} & {\small{}35.7} & {\small{}21.6} & {\small{}61.2} & {\small{}52.7} & {\small{}43.4} & {\small{}80.0} & {\small{}75.2} & {\small{}67.6} & {\small{}36.3} & {\small{}37.5} & {\small{}19.8}\tabularnewline
\hline 
{\small{}CRLC} & \textbf{\small{}79.9} & \textbf{\small{}67.9} & \textbf{\small{}63.4} & \textbf{\small{}42.5} & \textbf{\small{}41.6} & \textbf{\small{}26.3} & \textbf{\small{}81.8} & \textbf{\small{}72.9} & \textbf{\small{}68.2} & {\small{}85.4} & \textbf{\small{}83.1} & {\small{}75.9} & \textbf{\small{}46.1} & \textbf{\small{}48.4} & \textbf{\small{}29.7}\tabularnewline
\hline 
\end{tabular}
\par\end{centering}
\caption{End-to-end clustering results on 5 standard image datasets. Due to
space limit, we only show the means of the results. For the standard
deviations, please refer to Appdx.~A.8.\label{tab:N2N-clustering-results}}
\end{table*}

\begin{table*}
\begin{centering}
{\footnotesize{}}%
\begin{tabular}{|c|cccc|cccc|cccc|}
\hline 
{\footnotesize{}ImageNet} & \multicolumn{4}{c|}{{\footnotesize{}50 classes}} & \multicolumn{4}{c|}{{\footnotesize{}100 classes}} & \multicolumn{4}{c|}{{\footnotesize{}200 classes}}\tabularnewline
\hline 
{\footnotesize{}Metric} & {\footnotesize{}ACC} & {\footnotesize{}ACC5} & {\footnotesize{}NMI} & {\footnotesize{}ARI} & {\footnotesize{}ACC} & {\footnotesize{}ACC5} & {\footnotesize{}NMI} & {\footnotesize{}ARI} & {\footnotesize{}ACC} & {\footnotesize{}ACC5} & {\footnotesize{}NMI} & {\footnotesize{}ARI}\tabularnewline
\hline 
\hline 
{\footnotesize{}K-means \cite{van2020scan}} & {\footnotesize{}65.9} & {\footnotesize{}-} & {\footnotesize{}77.5} & {\footnotesize{}57.9} & {\footnotesize{}59.7} & {\footnotesize{}-} & {\footnotesize{}76.1} & {\footnotesize{}50.8} & {\footnotesize{}52.5} & {\footnotesize{}-} & {\footnotesize{}75.5} & {\footnotesize{}43.2}\tabularnewline
\hline 
{\footnotesize{}SCAN \cite{van2020scan}} & {\footnotesize{}75.1} & {\footnotesize{}91.9} & {\footnotesize{}80.5} & \textbf{\footnotesize{}63.5} & {\footnotesize{}66.2} & {\footnotesize{}88.1} & {\footnotesize{}78.7} & {\footnotesize{}54.4} & {\footnotesize{}56.3} & {\footnotesize{}80.3} & {\footnotesize{}75.7} & {\footnotesize{}44.1}\tabularnewline
\hline 
\hline 
{\footnotesize{}two-stage CRLC} & \textbf{\footnotesize{}75.4} & \textbf{\footnotesize{}93.3} & \textbf{\footnotesize{}80.6} & {\footnotesize{}63.4} & \textbf{\footnotesize{}66.7} & \textbf{\footnotesize{}88.3} & \textbf{\footnotesize{}79.2} & \textbf{\footnotesize{}55.0} & \textbf{\footnotesize{}57.9} & \textbf{\footnotesize{}80.6} & \textbf{\footnotesize{}76.4} & \textbf{\footnotesize{}45.9}\tabularnewline
\hline 
\end{tabular}{\footnotesize\par}
\par\end{centering}
\caption{Two-stage clustering results on ImageNet50/100/200.\label{tab:Two-stage-results-ImageNet}}
\end{table*}

\paragraph{Dataset}

We evaluate our proposed method on 5 standard datasets for image clustering
which are CIFAR10/20 \cite{krizhevsky2009learning}, STL10 \cite{coates2011analysis},
ImageNet10 \cite{deng2009imagenet,chang2017deep}, and ImageNet-Dogs
\cite{deng2009imagenet,chang2017deep}, and on 3 big ImageNet subsets
namely ImageNet50/100/200 with 50/100/200 classes, respectively \cite{deng2009imagenet,van2020scan}.
A description of these datasets is given in Appdx.~A.6. Our data
augmentation setting follows \cite{he2020momentum,wu2018unsupervised}.
We first randomly crop images to a desirable size (32$\times$32 for
CIFAR, 96$\times$96 for STL10, and 224$\times$224 for ImageNet subsets).
Then, we perform random horizontal flip, random color jittering, and
random grayscale conversion. For datasets which are ImageNet subsets,
we further apply Gaussian blurring at the last step \cite{chen2020simple}.
Similar to previous works \cite{chang2017deep,ji2019invariant,huang2020deep},
both the training and test sets are used for CIFAR10, CIFAR20 and
STL10 while only the training set is used for other datasets. We also
provide results where only the training set is used for CIFAR10, CIFAR20
and STL10 in Appdx.~A.8. For STL10, 100,000 auxiliary unlabeled samples
are additionally used to train the ``representation learning'' head.
However, when training the ``clustering'' head, these auxiliary
samples are not used since their classes may not appear in the training
set.

\paragraph{Model architecture and training setups}

Following previous works \cite{huang2020deep,ji2019invariant,van2020scan,zhong2020deep},
we adopt ResNet34 and ResNet50 \cite{he2016deep} as the backbone
network when working on the 5 standard datasets and on the 3 big ImageNet
subsets, respectively. The ``representation learning'' head (RL-head)
and the ``clustering'' head (C-head) are two-layer neural networks
with ReLU activations. The length of the output vector of the RL-head
is 128. The temperature $\tau$ (Eq.~\ref{eq:contrast_continuous_1})
is fixed at 0.1. To reduce variance in learning, we train our model
with 10 C-subheads\footnote{The final $\Loss_{\text{cluster}}$ in Eq.~\ref{eq:contrast_categorical_2}
is the average of $\Loss_{\text{cluster}}$ of these C-subheads.} similar to \cite{ji2019invariant}. This only adds little extra computation
to our model. However, unlike \cite{huang2020deep,ji2019invariant,zhong2020deep},
we do \emph{not} use an auxiliary ``over-clustering'' head to exploit
additional information from data since we think our RL-head can do
that effectively. 

Training setups for end-to-end and two-stage clustering are provided
in Appdx.~A.7.

\paragraph{Evaluation metrics}

We use three popular clustering metrics namely Accuracy (ACC), Normalized
Mutual Information (NMI), Adjusted Rand Index (ARI) for evaluation.
For unlabeled data, ACC is computed via the Kuhn-Munkres algorithm.
All of these metrics scale from 0 to 1 and higher values indicate
better performance. In this work, we convert the {[}0, 1{]} range
into percentage.

\subsection{Clustering\label{subsec:Results-for-clustering}}

\subsubsection{End-to-end training}

Table \ref{tab:N2N-clustering-results} compares the performance of
our proposed CRLC with a wide range of state-of-the-art deep clustering
methods. CRLC clearly outperforms all baselines by a large margin
on most datasets. For example, in term of clustering accuracy (ACC),
our method improves over the best baseline (DRC \cite{zhong2020deep})
by 5-7\% on CIFAR10/20, STL10, and ImageNet-Dogs. Gains are even larger
if we compare with methods that do not explicitly learn representations
such as PICA \cite{huang2020deep} and IIC \cite{ji2019invariant}.
CRLC only performs worse than DRC on ImageNet10, which we attribute
to our selection of hyperparameters. In addition, even when only
the ``clustering'' head is used, our method still surpasses most
of the baselines (e.g., DCCM, IIC). These results suggest that: i)
we can learn semantic clusters from data just by minimizing the probability
contrastive loss, and ii) combining with contrastive representation
learning improves the quality of the cluster assignment.

To have a better insight into the performance of CRLC, we visualize
some success and failure cases in Fig.~\ref{fig:STL10-samples-w-neighbors}
(and also in Appdx.~A.11). We see that samples predicted correctly
with high confidence are usually representative for the cluster they
belong to. It suggests that CRLC has learned coarse-grained patterns
that separate objects at the cluster level. Besides, CRLC has also
captured fine-grained instance-level information, thus, is able to
find nearest neighbors with great similarities in shape, color and
texture to the original image. Another interesting thing from Fig.~\ref{fig:STL10-samples-w-neighbors}
is that the predicted label of a sample is often strongly correlated
with that of the majority of its neighbors. It means that: i) CRLC
has learned a smooth mapping from images to cluster assignments, and
ii) CRLC tends to make ``collective'' errors (the first and third
rows in Fig.~\ref{fig:STL10-neighbors-false-positive}). Other kinds
of errors may come from the closeness between classes (e.g., horse
vs. dog), or from some adversarial signals in the input (e.g., the
second row in Fig.~\ref{fig:STL10-neighbors-false-negative}). Solutions
for fixing these errors are out of scope of this paper and will be
left for future work.

\begin{figure*}
\begin{centering}
\subfloat[correct\label{fig:STL10-neighbors-correct}]{\begin{centering}
\includegraphics[height=0.4\textwidth]{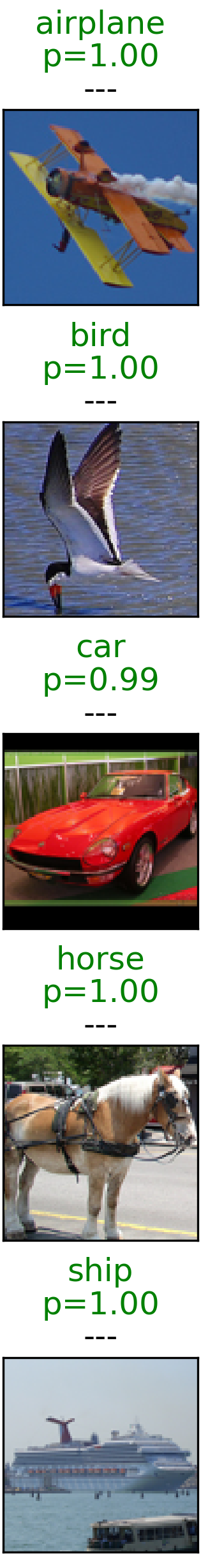}\hspace{0.01\textwidth}\includegraphics[height=0.4\textwidth]{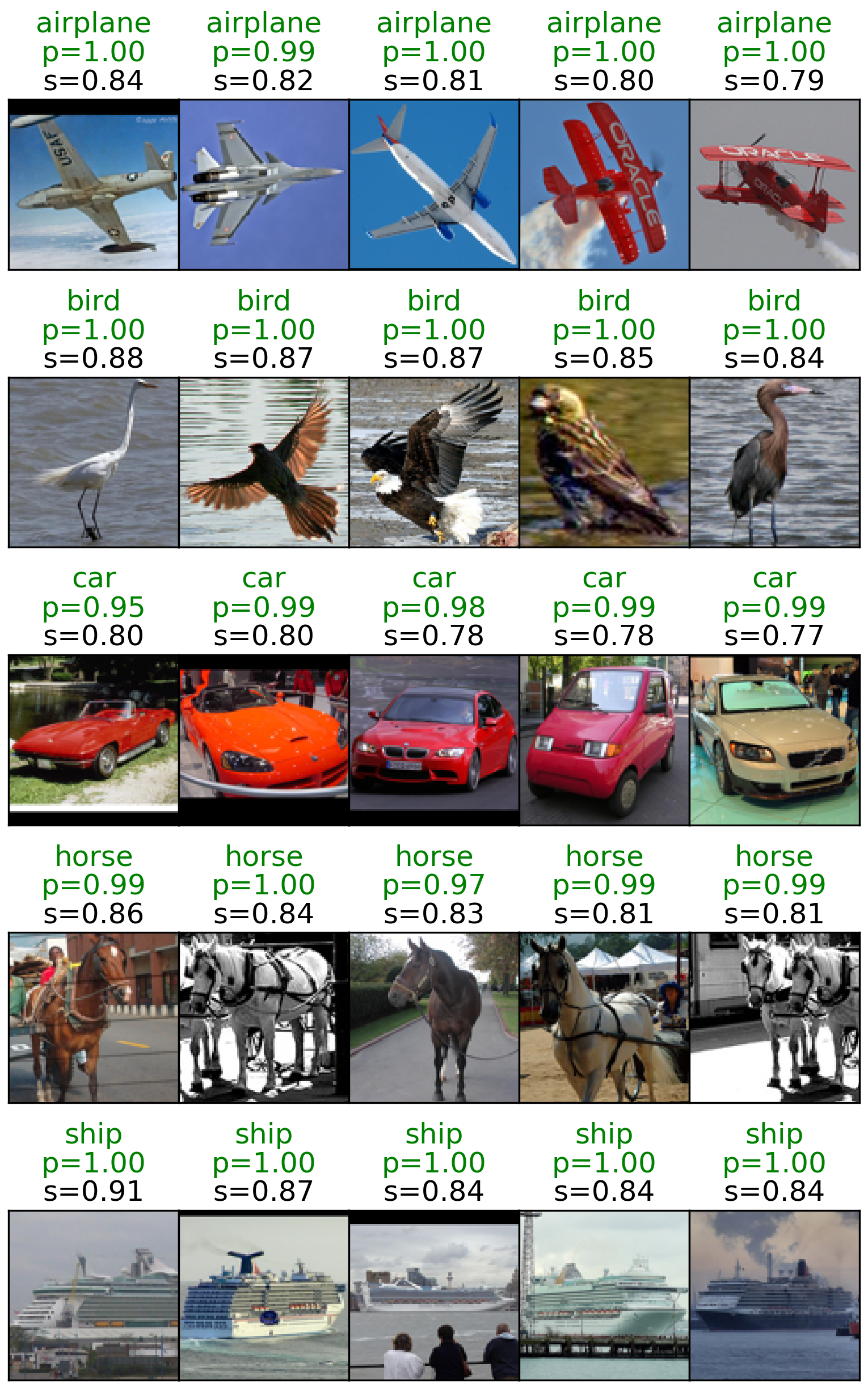}
\par\end{centering}
}\hspace{0.01\textwidth}\subfloat[false negative\label{fig:STL10-neighbors-false-negative}]{\begin{centering}
\includegraphics[height=0.4\textwidth]{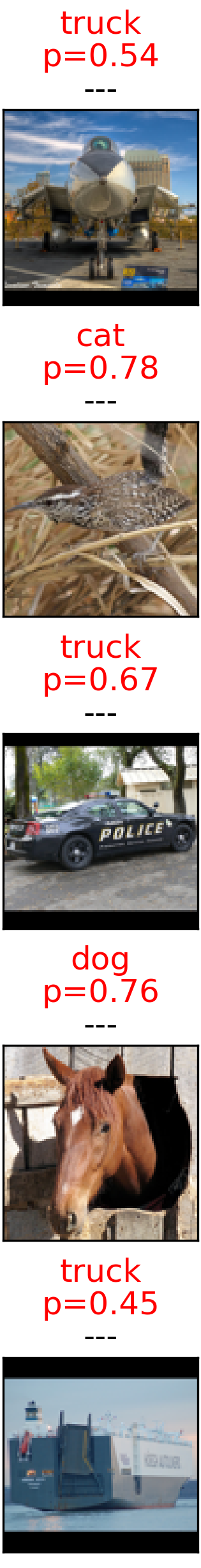}\hspace{0.01\textwidth}\includegraphics[height=0.4\textwidth]{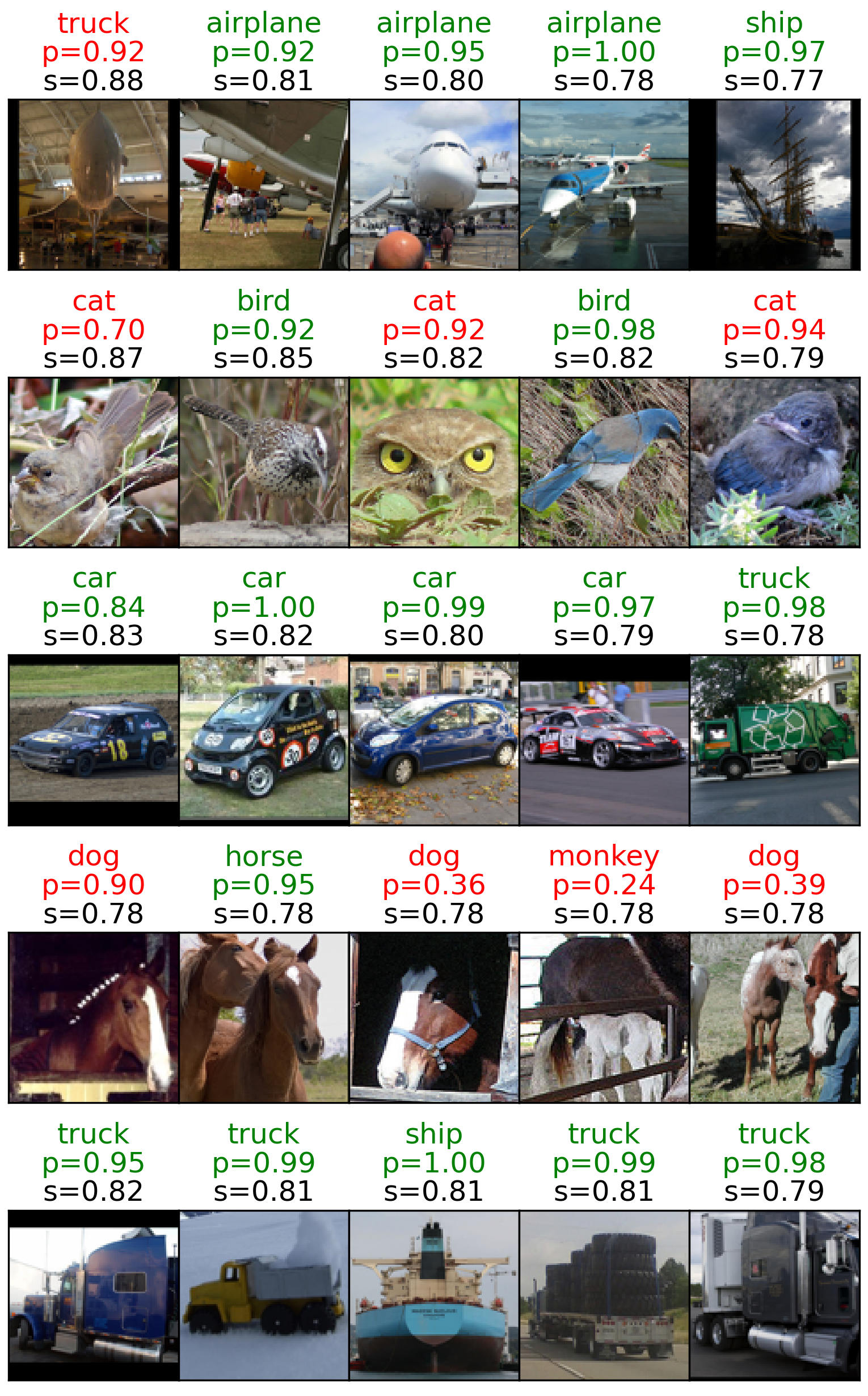}
\par\end{centering}
}\hspace{0.01\textwidth}\subfloat[false positive\label{fig:STL10-neighbors-false-positive}]{\begin{centering}
\includegraphics[height=0.4\textwidth]{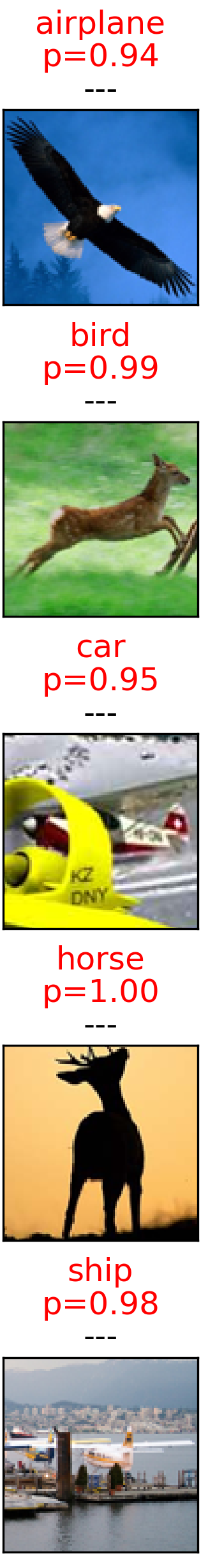}\hspace{0.01\textwidth}\includegraphics[height=0.4\textwidth]{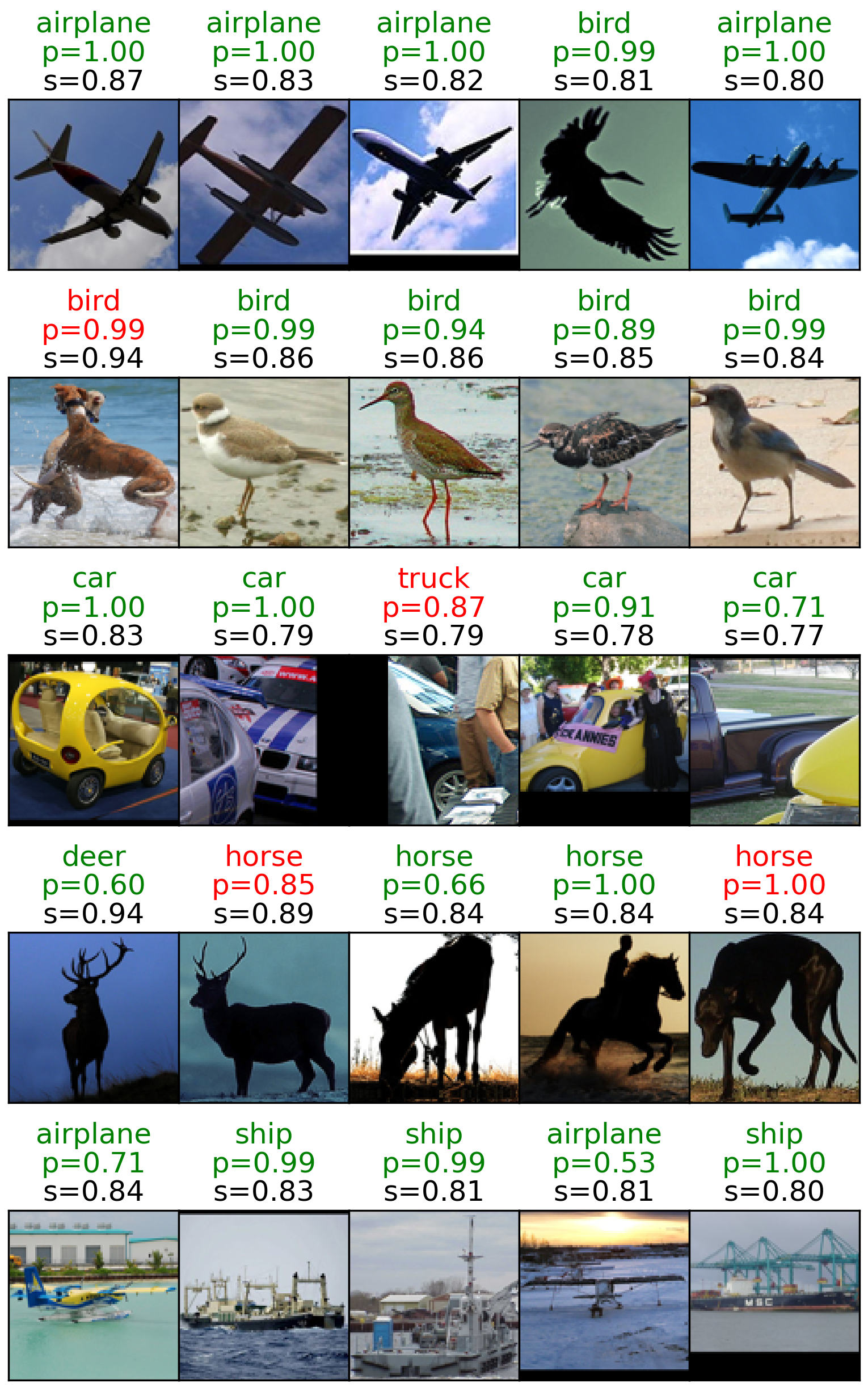}
\par\end{centering}
}
\par\end{centering}
\caption{STL10 samples of 5 classes correctly (green) and incorrectly (red)
predicted by CRLC. For each subplot, we show reference samples on
the leftmost column and their nearest neighbors on the right. Neighbors
are retrieved based on the normalized cosine similarity (``s'')
between the feature vectors of two samples. ``p'' denotes the confidence
probability.\label{fig:STL10-samples-w-neighbors}}
\end{figure*}

\subsubsection{Two-stage training\label{subsec:Two-stage-training}}

Although CRLC is originally proposed as an end-to-end clustering algorithm,
it can be easily extended to a two-stage clustering algorithm similar
to SCAN \cite{van2020scan}. To do that, we first pretrain the RL-head
and the backbone network with $\Loss_{\text{FC}}$ (Eq.~\ref{eq:contrast_continuous_2}).
Next, for every sample in the training data, we find a set of $K$
nearest neighbors based on the cosine similarity between feature vectors
produced by the pretrained network. In the second stage, we train
the C-head by minimizing $\Loss_{\text{cluster}}$ (Eq.~\ref{eq:contrast_categorical_2})
with the positive pair consisting of a sample and its neighbor drawn
from a set of $K$ nearest neighbors. We call this variant of CRLC
``two-stage'' CRLC. In fact, we did try training both the C-head
and the RL-head in the second stage by minimizing $\Loss_{\text{CRLC}}$
but could not achieve good results compared to training only the C-head.
We hypothesize that finetuning the RL-head causes the model to capture
too much fine-grained information and ignore important cluster-level
information, which hurts the clustering performance.

\begin{table}
\begin{centering}
{\footnotesize{}}%
\begin{tabular}{|c|c|c|c|}
\hline 
{\footnotesize{}Dataset} & \multicolumn{3}{c|}{{\footnotesize{}CIFAR10}}\tabularnewline
\hline 
{\footnotesize{}Labels} & {\footnotesize{}10} & {\footnotesize{}20} & {\footnotesize{}40}\tabularnewline
\hline 
\hline 
{\footnotesize{}MixMatch \cite{berthelot2019mixmatch}} & {\footnotesize{}-} & {\footnotesize{}-} & {\footnotesize{}47.54$\pm$11.50}\tabularnewline
\hline 
{\footnotesize{}UDA \cite{xie2019unsupervised}} & {\footnotesize{}-} & {\footnotesize{}-} & {\footnotesize{}29.05$\pm$5.93}\tabularnewline
\hline 
{\footnotesize{}ReMixMatch \cite{berthelot2019remixmatch}} & {\footnotesize{}-} & {\footnotesize{}-} & {\footnotesize{}19.10$\pm$9.64}\tabularnewline
\hline 
\hline 
{\footnotesize{}ReMixMatch$^{\dagger}$}\tablefootnote{{\footnotesize{}https://github.com/google-research/remixmatch}} & {\footnotesize{}59.86$\pm$9.34} & {\footnotesize{}41.68$\pm$8.15} & {\footnotesize{}28.31$\pm$6.72}\tabularnewline
\hline 
\hline 
{\footnotesize{}CRLC-semi} & \textbf{\footnotesize{}46.75$\pm$8.01} & \textbf{\footnotesize{}29.81$\pm$1.18} & \textbf{\footnotesize{}19.87$\pm$0.82}\tabularnewline
\hline 
\end{tabular}{\footnotesize\par}
\par\end{centering}
\caption{Classification errors on CIFAR10. Lower values are better. Results
of baselines are taken from \cite{sohn2020fixmatch}. $^{\dagger}$:
Results obtained from external implementations of models.\label{tab:Semi-supervised-learning-results}}
\end{table}

In Table~\ref{tab:Two-stage-results-ImageNet}, we show the clustering
results of ``two-stage'' CRLC on ImageNet50/100/200. Results on
CIFAR10/20 and STL10 are provided in Appdx.~A.9. For fair comparison
with SCAN, we use the same settings as in \cite{van2020scan} (details
in Appdx.~A.7). It is clear that ``two-stage'' CRLC outperforms
SCAN on all datasets. A possible reason is that besides pushing neighboring
samples close together, our proposed probability contrastive loss
also pulls away samples that are not neighbors (in the negative pairs)
while the SCAN's loss does not. Thus, by experiencing more pairs of
samples, our model is likely to form better clusters.

\subsection{Semi-supervised Learning\label{subsec:Semi-supervised-Learning-Experiments}}

Given the good performance of CRLC on clustering, it is natural to
ask whether this model also performs well on semi-supervised learning
(SSL) or not. To adapt for this new task, we simply train CRLC with
the new objective $\Loss_{\text{CLRC-semi}}$ (Eq.~\ref{eq:Loss_CRLC_semi}).
The model architecture and training setups remain almost the same
(changes in Appdx.~A.13).

\begin{figure*}
\begin{centering}
\includegraphics[height=0.2\textwidth]{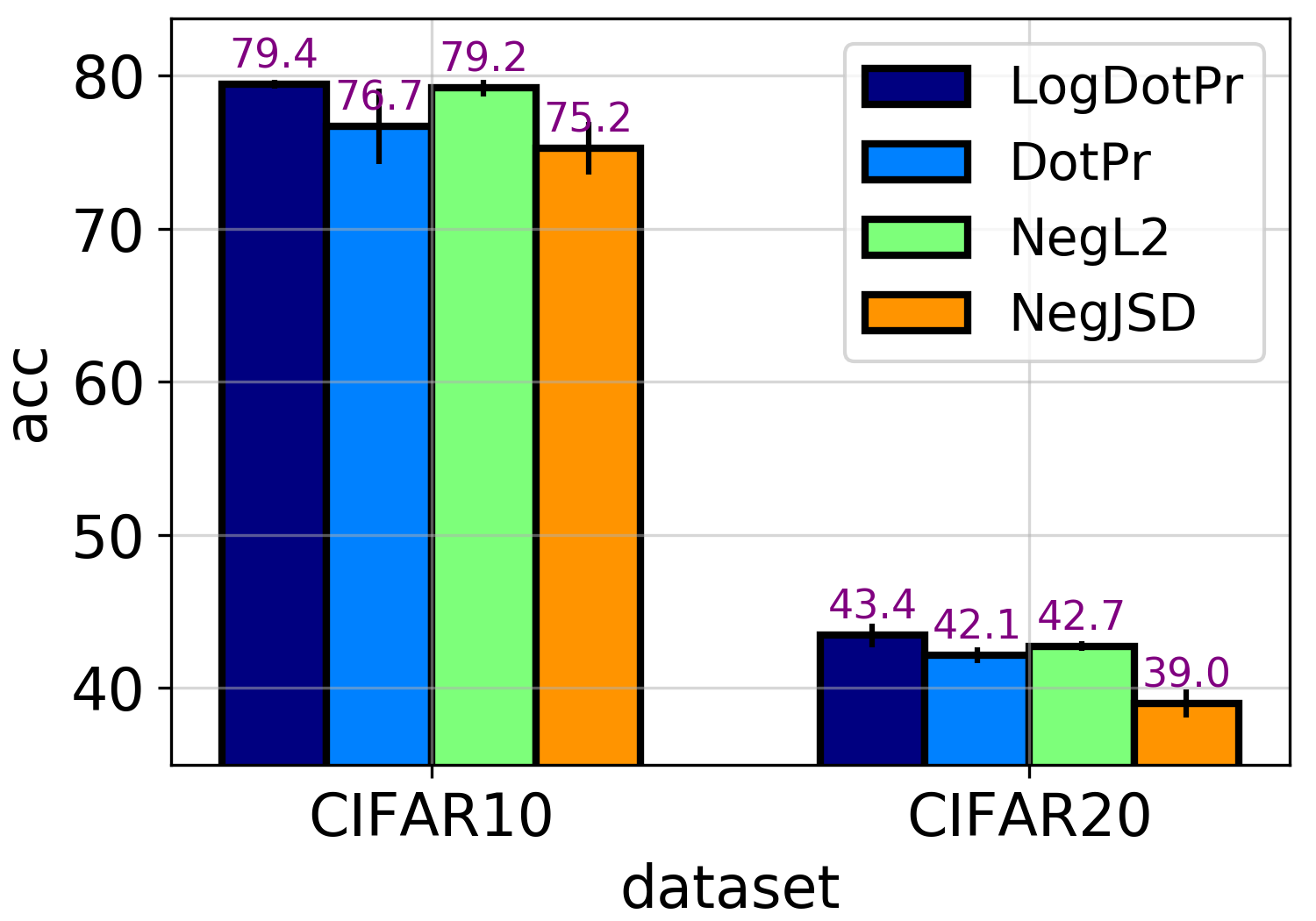}\hspace*{0.01\textwidth}\includegraphics[height=0.2\textwidth]{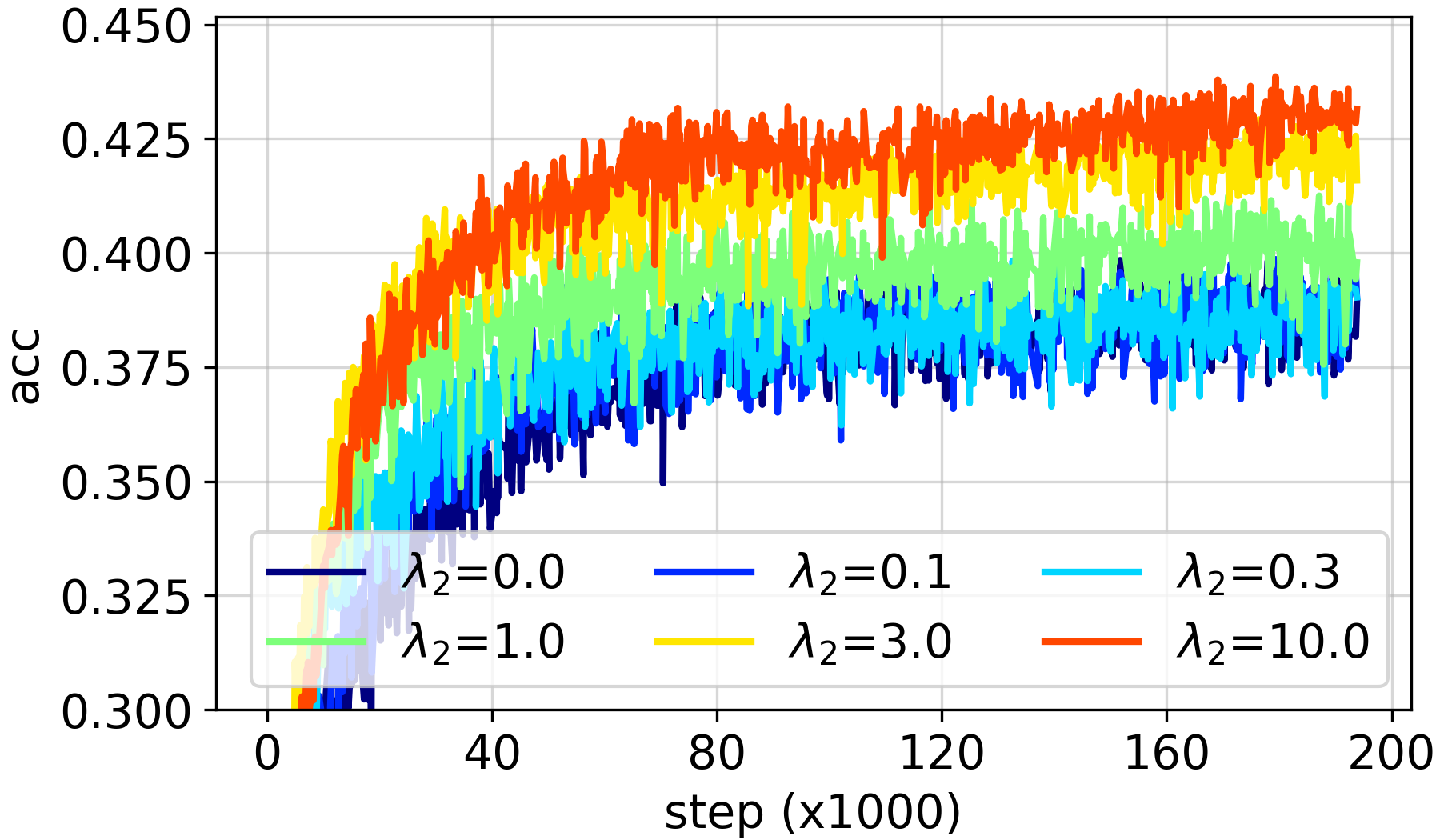}\hspace*{0.01\textwidth}\includegraphics[height=0.2\textwidth]{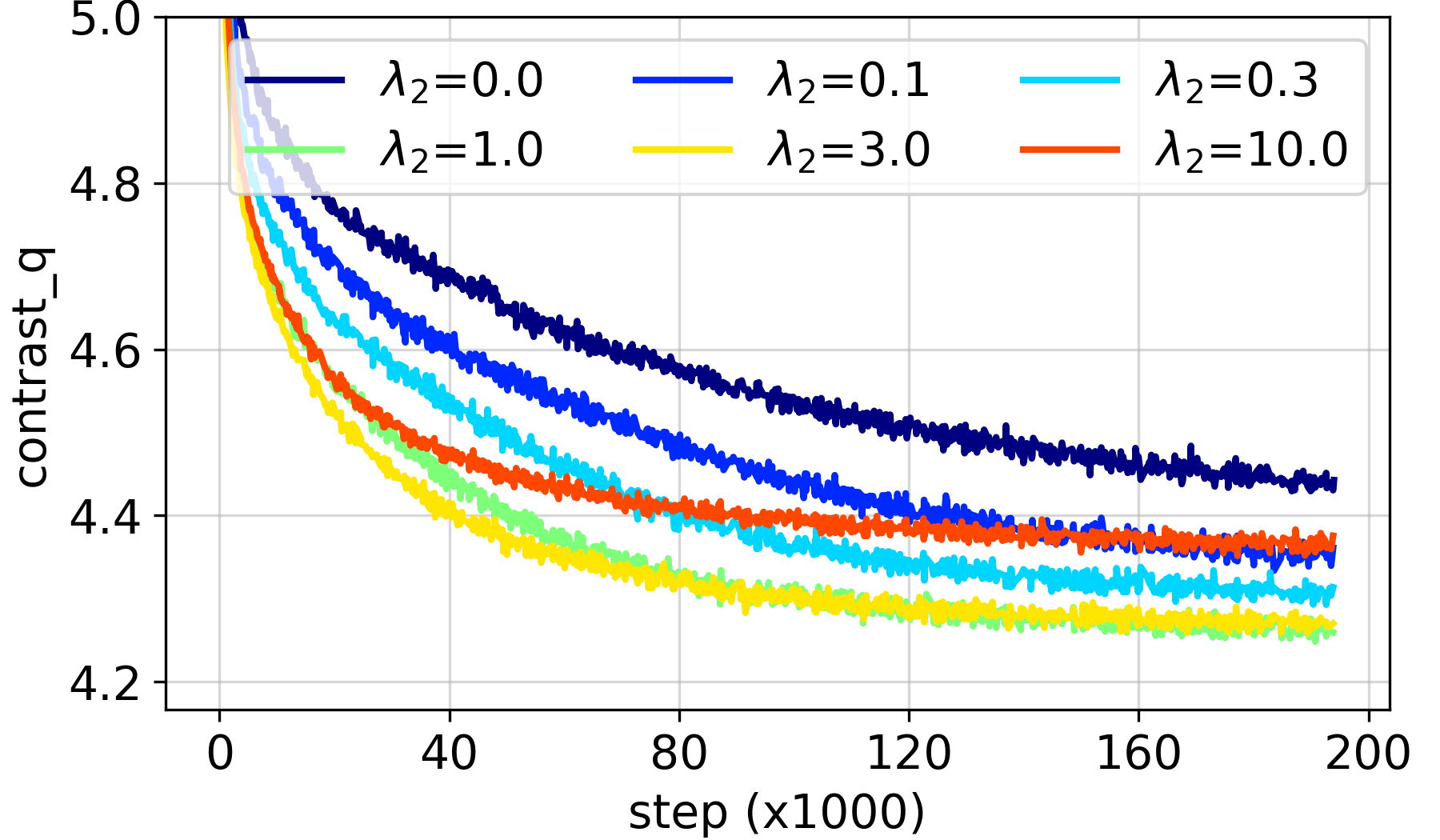}
\par\end{centering}
\caption{\textbf{Left}: Clustering accuracies of CRLC w.r.t. different critics
on CIFAR10/20 (training set only). \textbf{Middle}, \textbf{Right}:
Accuracy and $\protect\Loss_{\text{PC}}$ curves of CRLC on CIFAR20
w.r.t. different coefficients of $\protect\Loss_{\text{FC}}$ ($\lambda_{2}$
in Eq.~\ref{eq:Loss_CRLC}).\label{fig:bar_n_2curves}}
\end{figure*}

\begin{figure*}
\begin{centering}
\includegraphics[height=0.21\textwidth]{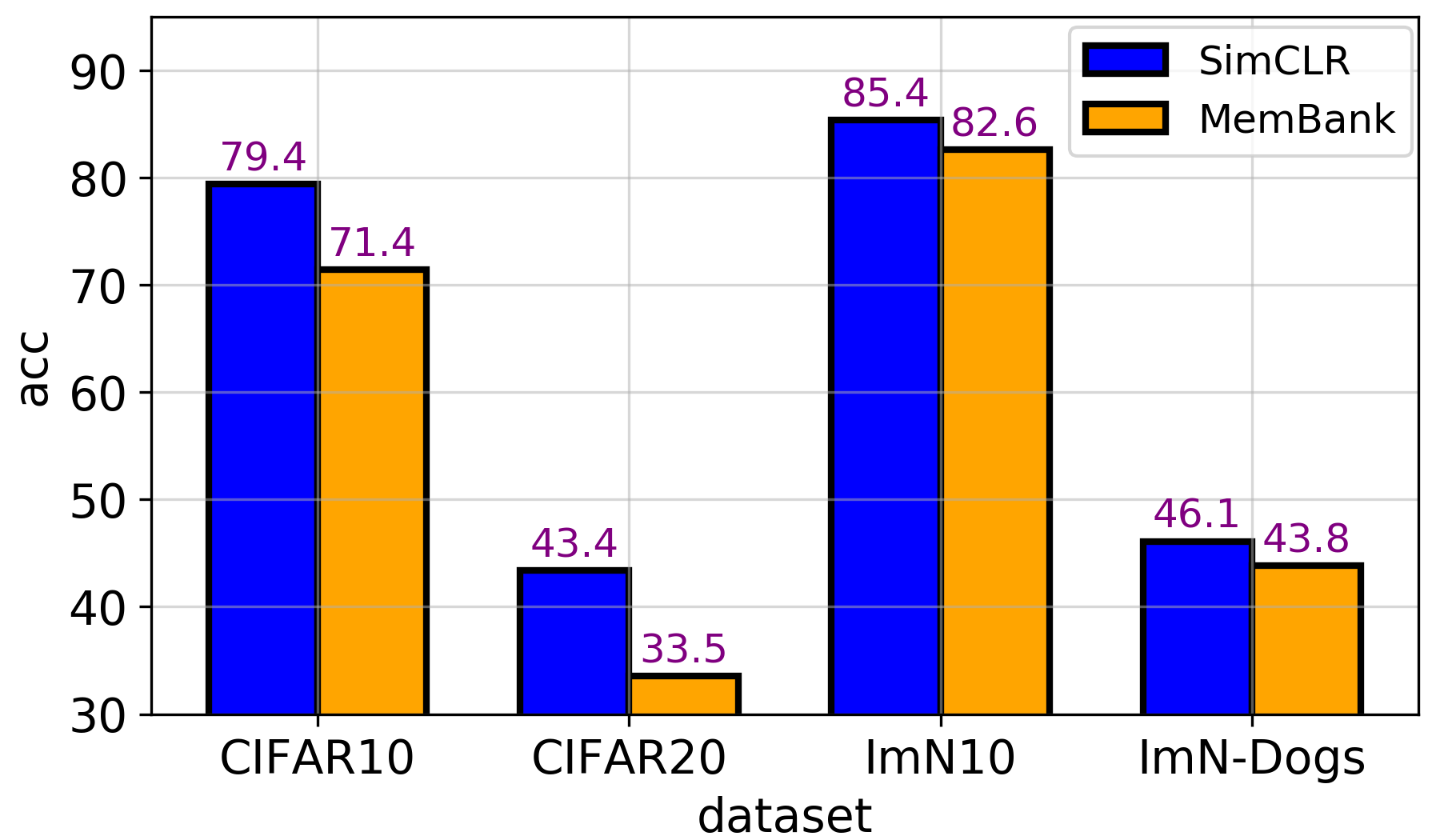}\hspace*{0.01\textwidth}\includegraphics[height=0.21\textwidth]{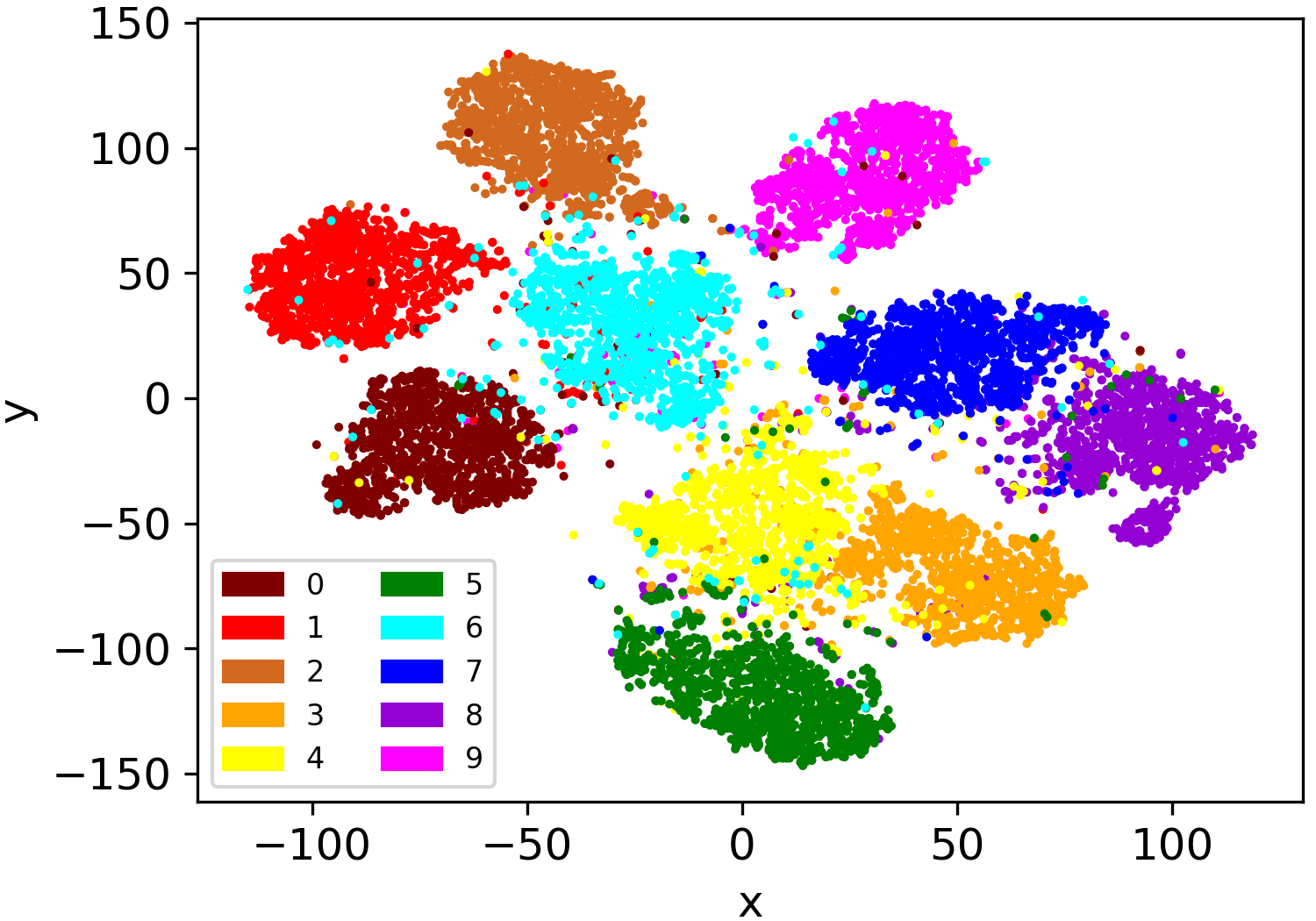}\hspace*{0.01\textwidth}\includegraphics[height=0.21\textwidth]{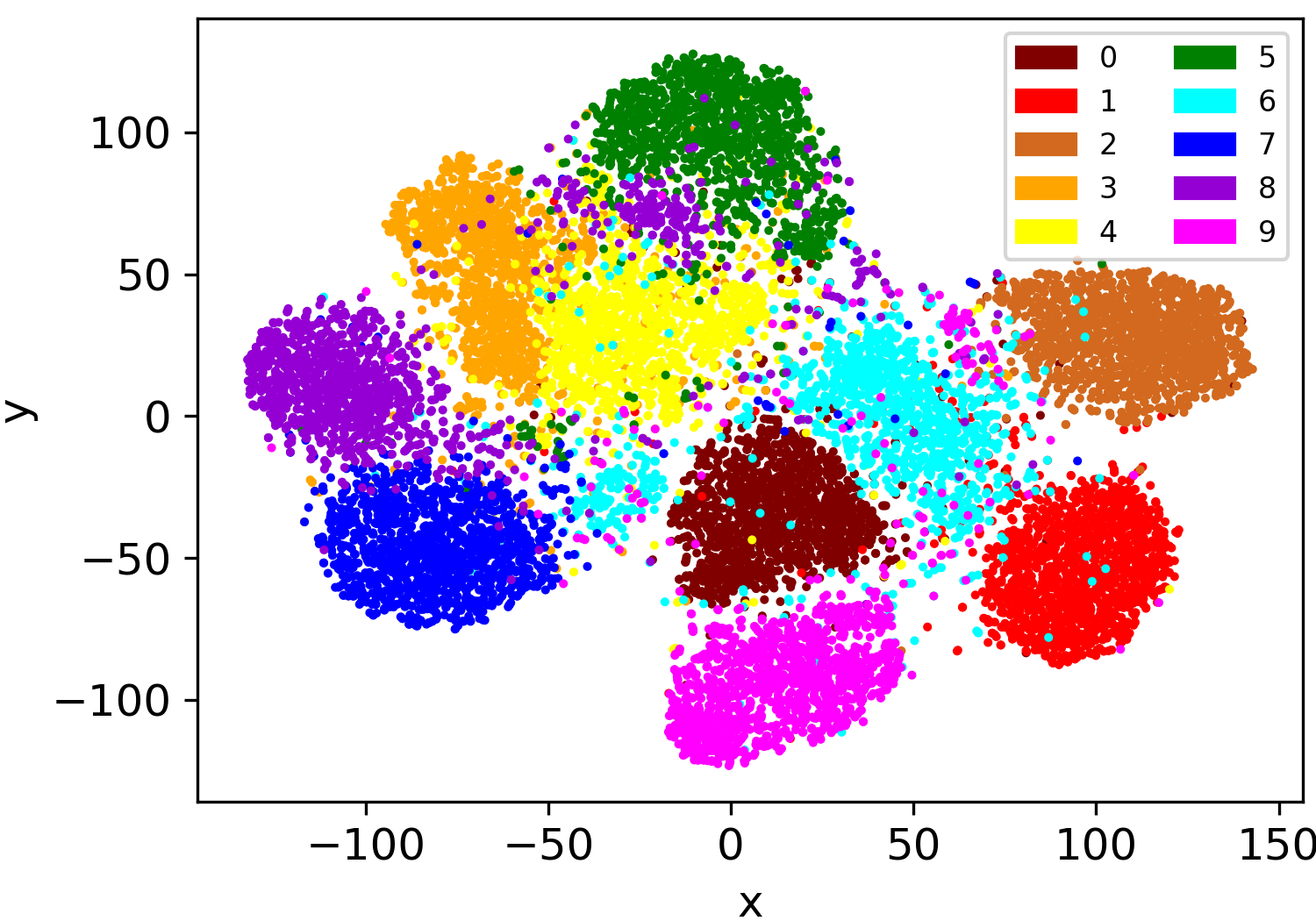}
\par\end{centering}
\caption{\textbf{Left}: Clustering accuracies of CRLC w.r.t. SimCLR \cite{chen2020simple}
and MemoryBank \cite{wu2018unsupervised} implementations. For CIFAR10/20,
only the training set is used. \textbf{Middle}, \textbf{Right}: tSNE
visualizations of the feature vectors learned by CRLC and SimCLR on
the ImageNet10 train set, repectively.\label{fig:bar_n_2embeds}}
\end{figure*}

From Table \ref{tab:Semi-supervised-learning-results}, we see that
CRLC-semi, though is not designed especially for SSL, significantly
outperforms many state-of-the-art SSL methods (brief discussion in
Appdx.~A.12). For example, CRLC-semi achieves about 30\% and 10\%
lower error than MixMatch \cite{berthelot2019mixmatch} and UDA \cite{xie2019unsupervised}
respectively on CIFAR10 with 4 labeled samples per class. Interestingly,
the power of CRLC-semi becomes obvious when the number of labeled
data is pushed to the limit. While most baselines cannot work with
1 or 2 labeled samples per class, CRLC-semi still performs \emph{consistently
well} with \emph{very low standard deviations}. We hypothesize the
reason is that CRLC-semi, via minimizing $\Loss_{\text{FC}}$, models
the ``smoothness'' of data better than the SSL baselines. For more
results on SSL, please check Appdx.~A.14.

\subsection{Ablation Study}

\paragraph{Comparison of different critics in the probability contrastive loss\label{par:Comparison-of-different-critics}}

In Fig.~\ref{fig:bar_n_2curves} left, we show the performance of
CRLC on CIFAR10 and CIFAR20 w.r.t. different critic functions. Apparently,
the theoretically sound ``log-of-dot-product'' critic (Eq.~\ref{eq:log_dot_prod_critic})
gives the best results. The ``negative-L2-distance'' critic is slightly
worse than the ``log-of-dot-product'' critic while the ``dot-product''
and the ``negative-JS-divergence'' critics are the worst.

\paragraph{Contribution of the feature contrastive loss\label{par:Contribution-of-the_LossFC}}

We investigate by how much our model's performance will be affected
if we change the coefficient of $\Loss_{\text{FC}}$ ($\lambda_{2}$
in Eq.~\ref{eq:Loss_CRLC}) to different values. Results on CIFAR20
are shown in Fig.~\ref{fig:bar_n_2curves} middle, right. Interestingly,
minimizing both $\Loss_{\text{PC}}$ and $\Loss_{\text{FC}}$ simultaneously
results in lower values of $\Loss_{\text{PC}}$ than minimizing only
$\Loss_{\text{PC}}$ ($\lambda_{2}=0$). It implies that $\Loss_{\text{FC}}$
provides the model with more information to form better clusters.
In order to achieve good clustering results, $\lambda_{2}$ should
be large enough relative to the coefficient of $\Loss_{\text{PC}}$
which is 1. However, too large $\lambda_{2}$ results in a high value
of $\Loss_{\text{PC}}$, which may hurts the model's performance.
For most datasets including CIFAR20, the optimal value of $\lambda_{2}$
is 10.

\paragraph{Nonparametric implementation of CRLC}

Besides using SimCLR \cite{chen2020simple}, we can also implement
the two contrastive losses in CRLC using MemoryBank \cite{wu2018unsupervised}
(Section~\ref{par:Implementing-LossPC}). This reduces the memory
storage by about 30\% and the training time by half (on CIFAR10 with
ResNet34 as the backbone and the minibatch size of 512). However,
MemoryBank-based CRLC usually takes longer time to converge and is
poorer than the SimCRL-based counterpart as shown in Fig.~\ref{fig:bar_n_2embeds}
left. The contributions of the number of negative samples and the
momentum coefficient to the performance of MemoryBank-based CRLC are
analyzed in Appdx.~A.10.2.

\paragraph{Mainfold visualization}

We visualize the manifold of the continous features learned by CRLC
in Fig.~\ref{fig:bar_n_2embeds} middle. We observe that CRLC usually
groups features into well-separate clusters. This is because the information
captured by the C-head has affected the RL-head. However, if the RL-head
is learned independently (e.g., in SimCLR), the clusters also emerge
but are usually close together (Fig.~\ref{fig:bar_n_2embeds} right).
Through both cases, we see the importance of contrastive representation
learning for clustering.

\section{Conclusion}

We proposed a novel clustering method named CRLC that exploits both
the fine-grained instance-level information and the coarse-grained
cluster-level information from data via a unified sample-oriented
contrastive learning framework. CRLC showed promising results not
only in clustering but also in semi-supervised learning. In the future,
we plan to enhance CRLC so that it can handle neighboring samples
in a principled way rather than just views. We also want to extend
CRLC to other domains (e.g., videos, graphs) and problems (e.g., object
detection).

\bibliographystyle{ieee_fullname}
\bibliography{ClusterInfoMax}

\cleardoublepage{}

\appendix

\section{Appendix}

\subsection{Possible critics for the probability contrastive loss\label{subsec:Possible-critics}}

We list here several possible critics that could be used in $\Loss_{\text{PC}}$.
If we simply consider a critic $f$ as a similarity measure of two
probabilities $p$ and $q$, $f$ could be the \emph{negative Jensen
Shannon (JS) divergence}\footnote{The JS divergence is chosen due to its symmetry. The negative sign
reflects the fact that $f$ is a similarity measure instead of a divergence.} between $p$ and $q$:
\begin{align}
f(p,q) & =-D_{\text{JS}}(p\|q)\label{eq:JS_critic1}\\
 & =-\frac{1}{2}\left(D_{\text{KL}}\left(p\big\|\frac{p+q}{2}\right)+D_{\text{KL}}\left(q\big\|\frac{p+q}{2}\right)\right)\label{eq:JS_critic2}
\end{align}
or the \emph{negative L2 distance} between $p$ and $q$:
\begin{equation}
f(p,q)=-\left\Vert p-q\right\Vert _{2}^{2}=-\sum_{c=1}^{C}{(p_{c}-q_{c})}^{2}\label{eq:L2_critic}
\end{equation}
In both cases, $f$ achieves its maximum value when $p=q$ and its
minimum value when $p$ and $q$ are different one-hot vectors.

We can also define $f$ as the \emph{dot product} of $p$ and $q$
as follows:
\begin{equation}
f(p,q)=p^{\top}q=\sum_{c=1}^{C}p_{c}q_{c}\label{eq:dot_prod_critic}
\end{equation}
However, the maximum value of this critic is no longer obtained when
$p=q$ but when $p$ and $q$ are \emph{the same one-hot vector} (check
Appdx.~\ref{subsec:Global-maxima-and-minima-dot-product-critic}
for details). It means that maximizing this critic encourages not
only the consistency between $p$ and $q$ but also the confidence
of $p$ and $q$.

\subsection{Global maxima and minima of the dot product critic for probabilities\label{subsec:Global-maxima-and-minima-dot-product-critic}}

\begin{figure*}
\begin{centering}
\subfloat[Dot-product critic\label{fig:Dot-product-critic-surface}]{\begin{centering}
\includegraphics[width=0.25\textwidth]{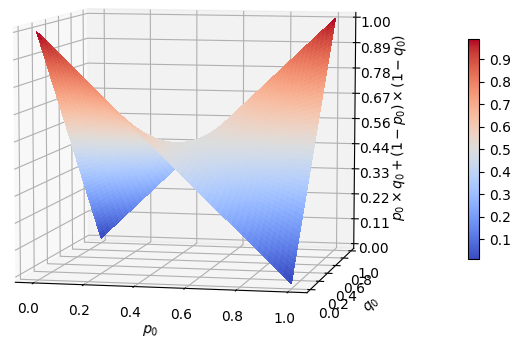}
\par\end{centering}
}\subfloat[Log-of-dot-product critic\label{fig:Log-of-dot-product-critic-surface}]{\begin{centering}
\includegraphics[width=0.25\textwidth]{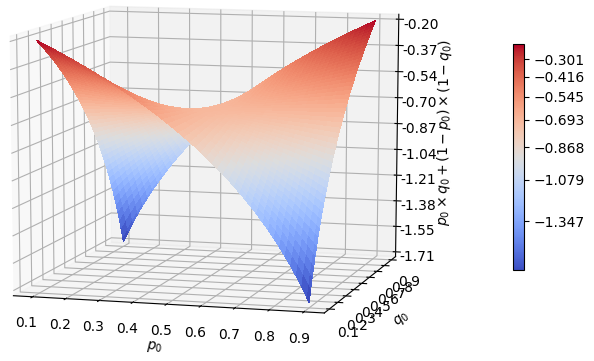}
\par\end{centering}
}\subfloat[Negative-L2-distance critic\label{fig:Neg-L2-dist-critic-surface}]{\begin{centering}
\includegraphics[width=0.25\textwidth]{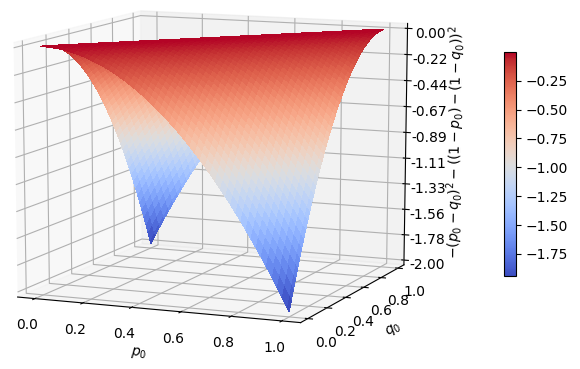}
\par\end{centering}
}\subfloat[Negative-JS-divergence critic]{\begin{centering}
\includegraphics[width=0.25\textwidth]{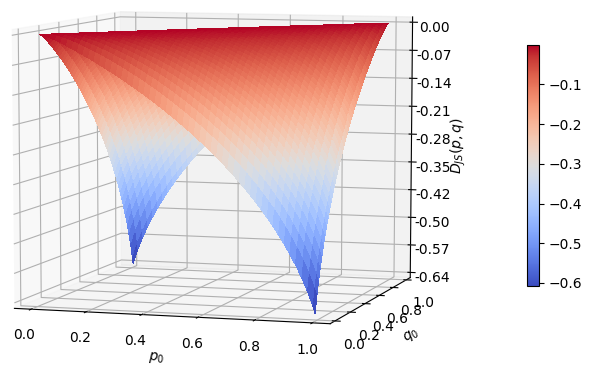}
\par\end{centering}
}
\par\end{centering}
\caption{The surfaces of different critics on probabilities in case of 2 classes.
\label{fig:critic_2D_surface}}
\end{figure*}

\begin{prop}
\textup{The dot product critic $f(p,q)=\sum_{c=1}^{C}p_{c}q_{c}$
achieves its global maximum value at $1$ when $p_{c}$ and $q_{c}$
are }\textup{\emph{the same one-hot vector}}\textup{, and its global
minimum value at $0$ when $p_{c}$ and $q_{c}$ are }\textup{\emph{different
one-hot vector}}\textup{.}
\end{prop}

\begin{proof}
Since $0\leq p_{c},q_{c}\leq1$, we have $\sum_{c=1}^{C}p_{c}q_{c}\geq0$.
This minimum value is achieved when $p_{c}q_{c}=0$ for all $c\in\{1,...,C\}$.
And because $\sum_{c=1}^{C}p_{c}=\sum_{c=1}^{C}q_{c}=1$, $p_{c}$
and $q_{c}$ must be different one-hot vectors.

In addition, we also have $\sum_{c=1}^{C}p_{c}q_{c}\leq\sum_{c=1}^{C}p_{c}=1$.
This maximum value is achieved when $p_{c}q_{c}=p_{c}$ or $p_{c}(q_{c}-1)=0$
for all $c\in\{1,...,C\}$, which means $p_{c}$ and $q_{c}$ must
be the same one-hot vectors.
\end{proof}
Since the gradient of $\sum_{c=1}^{C}p_{c}q_{c}$ w.r.t. $q_{c}$
is proportional to $p_{c}$, if we fix $p$ and only optimize $q$,
maximizing $\sum_{c=1}^{C}p_{c}q_{c}$ via gradient ascent will encourage
$q$ to be \emph{one-hot} at the component $k$ at which $p_{k}$
is the \emph{largest}. Similarly, minimizing $\sum_{c=1}^{C}p_{c}q_{c}$
via gradient descent will encourage $q$ to be \emph{one-hot} at the
component $k$ at which $p_{k}$ is the \emph{smallest}.

In case $p_{1}=...=p_{C}=\frac{1}{C}$, all the components of $q$
have similar gradients. Although it does not change the relative order
between the components of $q$ after update, it still push $q$ towards
the saddle point $\left(\frac{1}{C},...,\frac{1}{C}\right)$. However,
chance that models get stuck at this saddle point is tiny unless we
explicitly force it to happen (e.g., maximizing $H(q)$). 

For better understanding of the optimization dynamics, we visualize
the surface of $\sum_{c=1}^{C}p_{c}q_{c}$ with $C=2$ in Fig.~\ref{fig:Dot-product-critic-surface}.
$\log\left(\sum_{c=1}^{C}p_{c}q_{c}\right)$ has the same global optimal
values and surface as $\sum_{c=1}^{C}p_{c}q_{c}$

\subsection{Derivation of the InfoNCE lower bound\label{subsec:Derivation-of-the-InfoNCE-lowerbound}}

The variational lower bound of $I(X;Y)$ can be computed as follows:
\begin{align}
I(X;Y) & =\Expect_{p(x,y)}\left[\log\frac{p(x,y)}{p(x)p(y)}\right]\nonumber \\
 & =\Expect_{p(x,y)}\left[\log\frac{q_{\theta}(x,y)}{p(x)p(y)}\right]+D_{\text{KL}}\left(p(x,y)\|q_{\theta}(x,y)\right)\nonumber \\
 & \geq\Expect_{p(x,y)}\left[\log\frac{q_{\theta}(x,y)}{p(x)p(y)}\right]\label{eq:BA_lowerbound}
\end{align}
where $q_{\theta}(x,y)$ is the variational approximation of $p(x,y)$.

Following \cite{poole2019variational}, we assume that $q_{\theta}(x,y)$
belongs to the energy-based variational family that uses a critic
$f_{\theta}(x,y)$ and is scaled by the data density $p(x)p(y)$:
\begin{equation}
q_{\theta}(x,y)=\frac{p(x)p(y)e^{f_{\theta}(x,y)}}{\sum_{x,y}p(x)p(y)e^{f_{\theta}(x,y)}}=\frac{p(x)p(y)e^{f_{\theta}(x,y)}}{Z_{\theta}}\label{eq:q(x,y)}
\end{equation}
where $Z_{\theta}=\sum_{x,y}p(x)p(y)e^{f_{\theta}(x,y)}=\Expect_{p(x)p(y)}\left[e^{f_{\theta}(x,y)}\right]$
is the partition function which does not depend on $x$, $y$.

Since the optimal value of $q_{\theta}(x,y)$ is $q_{\theta}^{*}(x,y)=p(x,y)$,
we have:
\begin{align}
 & \frac{p(x)p(y)e^{f_{\theta}^{*}(x,y)}}{Z_{\theta}^{*}}=p(x,y)\nonumber \\
\Leftrightarrow & f_{\theta}^{*}(x,y)=\log Z_{\theta}^{*}+\log\frac{p(x,y)}{p(x)p(y)},\label{eq:optimal_critic}
\end{align}
which means the optimal value of $f_{\theta}(x,y)$ is proportional
to $\log\frac{p(x,y)}{p(x)p(y)}$.

Next, we will show that $f_{\theta}$ is the critic in the InfoNCE
lower bound. We start by rewriting the lower bound in Eq.~\ref{eq:BA_lowerbound}
using the formula of $q_{\theta}(x)$ in Eq.~\ref{eq:q(x,y)} as
follows:
\begin{align}
I(X;Y) & \geq\Expect_{p(x,y)}\left[\log\frac{e^{f_{\theta}(x,y)}}{Z_{\theta}}\right]\nonumber \\
 & =\Expect_{p(x,y)}\left[f_{\theta}(x,y)\right]-\log Z_{\theta}\label{eq:UBA_lowerbound}
\end{align}
Here, we encounter the intractable $\log Z_{\theta}$. To form a tractable
lower bound of $I(X;Y)$, we continue replacing $\log Z_{\theta}$
with its variational upper bound: 
\begin{equation}
\log Z_{\theta}\leq\frac{Z_{\theta}}{a_{\theta}}+\log a_{\theta}-1\label{eq:logZ_upperbound}
\end{equation}
where $a_{\theta}$ is the variational approximation of $Z_{\theta}$.
We should choose $a_{\theta}$ close to $Z_{\theta}$ so that the
variance of the bound in Eq.~\ref{eq:logZ_upperbound} is small.
Recalling that $Z_{\theta}=\Expect_{p(x)p(y)}\left[e^{f_{\theta}(x,y)}\right]$,
we define $a_{\theta}$ as follows:
\begin{equation}
a_{\theta}=\frac{1}{M}\sum_{i=1}^{M}e^{f_{\theta}(x_{i},y)}\label{eq:a_formula}
\end{equation}
where $x_{1},...,x_{M}$ are $M$ samples from $p(x)$. $a_{\theta}$
in Eq.~\ref{eq:a_formula} can be seen as a stochastic estimation
of $Z_{\theta}$ with $x$ sampled $M$ times more than $y$. Thus,
$\frac{Z_{\theta}}{a_{\theta}}\approx1$ and from Eq.~\ref{eq:logZ_upperbound},
we have $\log Z_{\theta}\leq\log a_{\theta}$. Apply this result to
Eq.~\ref{eq:UBA_lowerbound}, we have:
\begin{align}
I(X;Y) & \geq\Expect_{p(x,y)}\left[f_{\theta}(x,y)\right]-\log a_{\theta}\label{eq:InfoNCE_bound1}\\
 & =\Expect_{p(x_{2:M})}\Expect_{p(x_{1},y)}\left[f_{\theta}(x_{1},y)-\log\frac{1}{M}\sum_{i=1}^{M}e^{f_{\theta}(x_{i},y)}\right]\label{eq:InfoNCE_bound2}\\
 & =\Expect_{p(x_{1:M})p(y|x_{1})}\left[\log\frac{e^{f_{\theta}(x_{1},y)}}{\sum_{i=1}^{M}e^{f_{\theta}(x_{i},y)}}\right]+\log M\label{eq:InfoNCE_bound3}\\
 & \triangleq I_{\text{InfoNCE}}(X;Y)\label{eq:InfoNCE_bound4}
\end{align}
where Eq.~\ref{eq:InfoNCE_bound2} is obtained from Eq.~\ref{eq:InfoNCE_bound1}
by using the fact that $\Expect_{p(x,y)}\left[f_{\theta}(x,y)\right]=\Expect_{p(x_{2:M})}\Expect_{p(x_{1},y)}\left[f_{\theta}(x_{1},y)\right]$
and the assumption that the samples $x_{1},...,x_{M}$ and $y$ in
$a_{\theta}$ (Eq.~\ref{eq:a_formula}) are drawn from $p(x_{2:M})p(x_{1},y)$.

Combining with the result in Eq.~\ref{eq:optimal_critic}, we have
the optimal critic $f_{\theta}^{*}(x,y)$ in the InfoNCE lower bound
is proportional to $\log\frac{p(x,y)}{p(x)p(y)}=\log\frac{p(y|x)}{p(y)}$.
Since $p(y)$ does not depend on $x$ and will be cancelled by both
the nominator and denominator in Eq.~\ref{eq:InfoNCE_bound3}, $f_{\theta}^{*}(x,y)$
is, in fact, proportional to $\log p(y|x)$.

\subsection{Derivation of the scaled dot product critic in representation learning\label{subsec:Derivation-of-the-scale-dot-product-critic}}

Recalling that in contrastive representation learning, the critic
$f$ is defined as the scaled dot product between two unit-normed
feature vectors $\tilde{z}$, $z_{i}$:
\[
f(\tilde{x},x_{i})=\tilde{z}^{\top}z_{i}/\tau
\]
Interestingly, this formula of $f$ is accordant with the formula
of $f^{*}$ and is proportional to $\log p(\tilde{x}|x_{i})$. To
see why, let's assume that the distribution of $\tilde{z}$ given
$z_{i}$ is modeled by an isotropic Gaussian distribution with $z_{i}$
as the mean vector and $\tau I$ as the covariance matrix. Then, we
have:
\begin{align*}
f^{*} & \propto\log p(\tilde{x}|x_{i})\\
 & \approx\log p(\tilde{z}|z_{i})\\
 & \propto\log e^{-\frac{0.5}{\tau}\left\Vert \tilde{z}-z_{i}\right\Vert _{2}^{2}}\\
 & =-\frac{0.5}{\tau}\left(\left\Vert \tilde{z}\right\Vert _{2}^{2}-2\tilde{z}^{\top}z_{i}+\left\Vert z_{i}\right\Vert _{2}^{2}\right)\\
 & =\tilde{z}^{\top}z_{i}/\tau-1/\tau\\
 & \propto\tilde{z}^{\top}z_{i}/\tau
\end{align*}
where $\left\Vert \tilde{z}\right\Vert _{2}^{2}=\left\Vert z_{i}\right\Vert _{2}^{2}=1$
due to the fact that $\tilde{z}$ and $z_{i}$ are unit-normed vectors.

\subsection{Analysis of the gradient of $\protect\Loss_{\text{PC}}$\label{subsec:Analysis-of-the-gradient}}

Recalling that the probability contrastive loss $\Loss_{\text{PC}}$
for a sample $\tilde{x}$ with the ``log-of-dot-product'' critic
$f(p,q)=\log\left(p^{\top}q\right)$ is computed as follows:
\begin{align*}
\Loss_{\text{PC}} & =-\log\frac{e^{f(\tilde{q},q_{1})}}{\sum_{i=1}^{M}e^{f(\tilde{q},q_{i})}}\\
 & =-\log\left(\tilde{q}^{\top}q_{1}\right)+\log\sum_{i=1}^{M}\tilde{q}^{\top}q_{i}
\end{align*}
Because $\tilde{q}$ is always parametric while $q_{i}$ ($i\in\{1,...,M\}$)
can be either parametric (if $\Loss_{\text{PC}}$ is implemented via
the SimCLR framework \cite{chen2020simple}) or non-parametric (if
$\Loss_{\text{PC}}$ is implemented via the MemoryBank framework \cite{wu2018unsupervised}),
we focus on the gradient of $\Loss_{\text{PC}}$ back-propagating
through $\tilde{q}$. In practice, $\tilde{q}$ is usually implemented
by applying softmax to the logit vector $\tilde{u}\in\Real^{C}$:
\[
\tilde{q}_{c}=\frac{\exp\left(\tilde{u}_{c}\right)}{\sum_{k=1}^{C}\exp\left(\tilde{u}_{k}\right)}
\]
where $\tilde{q}_{c}$ denotes the $c$-th component of $\tilde{q}$.
Similarly, $q_{i,c}$ is the $c$-th component of $q_{i}$.

The gradient of $\Loss_{\text{PC}}$ w.r.t. $\tilde{u}_{c}$ is given
by:

\begin{equation}
\frac{\partial\Loss_{\text{PC}}}{\partial\tilde{u}_{c}}=-\frac{\partial}{\partial\tilde{u}_{c}}\log\left(\tilde{q}^{\top}q_{1}\right)+\frac{\partial}{\partial\tilde{u}_{c}}\log\sum_{i=1}^{M}\tilde{q}^{\top}q_{i}\label{eq:Grad_LossPC_1}
\end{equation}
The first term in Eq.~\ref{eq:Grad_LossPC_1} is equivalent to:
\begin{align}
 & -\frac{\partial}{\partial\tilde{u}_{c}}\log\left(\tilde{q}^{\top}q_{1}\right)\nonumber \\
\Leftrightarrow & \frac{-1}{\tilde{q}^{\top}q_{1}}\left(\frac{\partial}{\partial\tilde{u}_{c}}\left(\tilde{q}_{c}q_{1,c}\right)+\sum_{k\neq c}\frac{\partial}{\partial\tilde{u}_{c}}\left(\tilde{q}_{k}q_{1,k}\right)\right)\nonumber \\
\Leftrightarrow & \frac{-1}{\tilde{q}^{\top}q_{1}}\left(\tilde{q}_{c}(1-\tilde{q}_{c})q_{1,c}-\sum_{k\neq c}\tilde{q}_{c}\tilde{q}_{k}q_{1,k}\right)\nonumber \\
\Leftrightarrow & -\frac{1}{\sum_{k=1}^{C}\tilde{q}_{k}q_{1,k}}\left(\tilde{q}_{c}q_{1,c}-\tilde{q}_{c}\sum_{k=1}^{C}\tilde{q}_{k}q_{1,k}\right)\nonumber \\
\Leftrightarrow & \tilde{q}_{c}-\frac{\tilde{q}_{c}q_{1,c}}{\sum_{k=1}^{C}\tilde{q}_{k}q_{1,k}}\label{eq:neg_critic_grad}
\end{align}
And the second term in Eq.~\ref{eq:Grad_LossPC_1} is equivalent
to:
\begin{align*}
 & \frac{\partial}{\partial\tilde{u}_{c}}\log\sum_{i=1}^{M}\tilde{q}^{\top}q_{i}\\
\Leftrightarrow & \frac{1}{\sum_{i=1}^{M}\tilde{q}^{\top}q_{i}}\left(\sum_{i=1}^{M}\frac{\partial}{\partial\tilde{u}_{c}}\left(\tilde{q}^{\top}q_{i}\right)\right)\\
\Leftrightarrow & \frac{1}{\sum_{i=1}^{M}\tilde{q}^{\top}q_{i}}\left(\sum_{i=1}^{M}\left(\tilde{q}_{c}q_{i,c}-\tilde{q}_{c}\sum_{k=1}^{C}\tilde{q}_{k}q_{i,k}\right)\right)\\
\Leftrightarrow & \frac{1}{\sum_{i=1}^{M}\sum_{k=1}^{C}\tilde{q}_{k}q_{i,k}}\left(\sum_{i=1}^{M}\tilde{q}_{c}q_{i,c}-\tilde{q}_{c}\sum_{i=1}^{M}\sum_{k=1}^{C}\tilde{q}_{k}q_{i,k}\right)\\
\Leftrightarrow & \frac{\sum_{i=1}^{M}\tilde{q}_{c}q_{i,c}}{\sum_{i=1}^{M}\sum_{k=1}^{C}\tilde{q}_{k}q_{i,k}}-\tilde{q}_{c}
\end{align*}
Thus, we have:
\begin{equation}
\frac{\partial\Loss_{\text{PC}}}{\partial\tilde{u}_{c}}=\frac{\sum_{i=1}^{M}\tilde{q}_{c}q_{i,c}}{\sum_{i=1}^{M}\sum_{k=1}^{C}\tilde{q}_{k}q_{i,k}}-\frac{\tilde{q}_{c}q_{1,c}}{\sum_{k=1}^{C}\tilde{q}_{k}q_{1,k}}\label{eq:Grad_LossPC_final}
\end{equation}

We care about the second term in Eq.~\ref{eq:Grad_LossPC_final}
which is derived from the gradient of the critic $f(\tilde{q},q_{1})$
w.r.t. $\tilde{u}_{c}$ (the negative of the term in Eq.~\ref{eq:neg_critic_grad}).
We rewrite this gradient with simplified notations as follows:
\[
\frac{\partial f(q,p)}{\partial u_{c}}=\frac{q_{c}p_{c}}{\sum_{k=1}^{C}q_{k}p_{k}}-q_{c}
\]
where $u_{c}$ is the $c$-th logit of $q$. Since during training,
$q$ is encouraged to be one-hot (see Appdx.~\ref{subsec:Global-maxima-and-minima-dot-product-critic}),
the denominator may not be defined if we do not prevent $p$ from
being a different one-hot vector. However, even when the denominator
is defined, the update still does not happen as expected when $q$
is one-hot. To see why, let's consider a simple scenario in which
$q=[0,1,0]$ and $p=[0.998,0.001,0.001]$. Apparently, the denominator
is 0.001 $\neq$ 0. By maximizing $f(q,p)$, we want to push $q$
toward $p$. Thus, we expect that $\frac{\partial f}{\partial u_{1}}>0$
and $\frac{\partial f}{\partial u_{2}}<0$. However, the gradients
w.r.t. $u_{c}$ are 0s for all $c\in\{1,2,3\}$:

\begin{align*}
\frac{\partial f}{\partial u_{1}} & =\frac{0\times0.998}{0.001}-0=0\\
\frac{\partial f}{\partial u_{2}} & =\frac{1\times0.001}{0.001}-1=0\\
\frac{\partial f}{\partial u_{3}} & =\frac{0\times0.001}{0.001}-0=0
\end{align*}
The reason is that $q=[0,1,0]$ is a stationary point (minimum in
this case). This means once the model has set $q$ to be one-hot,
it tends to get stuck there and cannot escape \emph{regardless of
the value of} $p$. This problem is known in literature as the ``saturating
gradient'' problem. To alleviate this problem, we propose to smooth
out the values of $q$ and $p$ before computing the critic $f$:
\begin{align*}
q & =(1-\gamma)q+\gamma r\\
p & =(1-\gamma)p+\gamma r
\end{align*}
where $0\leq\gamma\leq1$ is the smoothing coefficient, which is set
to 0.01 if not otherwise specified; $r=\left(\frac{1}{C},...,\frac{1}{C}\right)$
is the uniform probability vector over classes. We also regularize
the value of $u_{c}$ to be within $[-25,25]$.

\begin{table}
\begin{centering}
{\scriptsize{}}%
\begin{tabular}{|c|c|c|c|c|c|}
\hline 
{\scriptsize{}Dataset} & {\scriptsize{}\#Train} & {\scriptsize{}\#Test} & {\scriptsize{}\#Extra} & {\scriptsize{}\#Classes} & {\scriptsize{}Image size}\tabularnewline
\hline 
\hline 
{\scriptsize{}CIFAR10} & {\scriptsize{}50,000} & {\scriptsize{}10,000} & {\scriptsize{}$\times$} & {\scriptsize{}10} & {\scriptsize{}32$\times$32$\times$3}\tabularnewline
\hline 
{\scriptsize{}CIFAR20} & {\scriptsize{}50,000} & {\scriptsize{}10,000} & {\scriptsize{}$\times$} & {\scriptsize{}20} & {\scriptsize{}32$\times$32$\times$3}\tabularnewline
\hline 
{\scriptsize{}STL10} & {\scriptsize{}5,000} & {\scriptsize{}8,000} & {\scriptsize{}100,000} & {\scriptsize{}10} & {\scriptsize{}96$\times$96$\times$3}\tabularnewline
\hline 
{\scriptsize{}ImageNet10} & {\scriptsize{}13,000} & {\scriptsize{}500} & {\scriptsize{}$\times$} & {\scriptsize{}10} & {\scriptsize{}224$\times$224$\times$3}\tabularnewline
\hline 
{\scriptsize{}ImageNet-Dogs} & {\scriptsize{}19,500} & {\scriptsize{}750} & {\scriptsize{}$\times$} & {\scriptsize{}15} & {\scriptsize{}224$\times$224$\times$3}\tabularnewline
\hline 
{\scriptsize{}ImageNet-50} & {\scriptsize{}64,274} & {\scriptsize{}2,500} & {\scriptsize{}$\times$} & {\scriptsize{}50} & {\scriptsize{}224$\times$224$\times$3}\tabularnewline
\hline 
{\scriptsize{}ImageNet-100} & {\scriptsize{}128,545} & {\scriptsize{}5,000} & {\scriptsize{}$\times$} & {\scriptsize{}100} & {\scriptsize{}224$\times$224$\times$3}\tabularnewline
\hline 
{\scriptsize{}ImageNet-200} & {\scriptsize{}256,558} & {\scriptsize{}10,000} & {\scriptsize{}$\times$} & {\scriptsize{}200} & {\scriptsize{}224$\times$224$\times$3}\tabularnewline
\hline 
\end{tabular}{\scriptsize\par}
\par\end{centering}
\caption{Details of the datasets used in this work.\label{tab:Datasets_Details}}
\end{table}

\begin{table*}
\begin{centering}
{\footnotesize{}}%
\begin{tabular}{|c|c|ccc|ccc|ccc|}
\hline 
\multicolumn{2}{|c|}{{\footnotesize{}Dataset}} & \multicolumn{3}{c|}{{\footnotesize{}CIFAR10}} & \multicolumn{3}{c|}{{\footnotesize{}CIFAR20}} & \multicolumn{3}{c|}{{\footnotesize{}STL10}}\tabularnewline
\hline 
\multicolumn{2}{|c|}{{\footnotesize{}Metric}} & {\footnotesize{}ACC} & {\footnotesize{}NMI} & {\footnotesize{}ARI} & {\footnotesize{}ACC} & {\footnotesize{}NMI} & {\footnotesize{}ARI} & {\footnotesize{}ACC} & {\footnotesize{}NMI} & {\footnotesize{}ARI}\tabularnewline
\hline 
\hline 
\multirow{2}{*}{{\footnotesize{}C-head only}} & {\footnotesize{}Train only} & \textbf{\footnotesize{}67.2$\pm$0.7} & {\footnotesize{}56.8$\pm$1.3} & \textbf{\footnotesize{}47.8$\pm$1.4} & \textbf{\footnotesize{}38.0$\pm$1.6} & \textbf{\footnotesize{}36.8$\pm$1.0} & \textbf{\footnotesize{}22.3$\pm$0.9} & {\footnotesize{}47.03$\pm$2.2} & {\footnotesize{}39.06$\pm$1.5} & {\footnotesize{}27.23$\pm$1.8}\tabularnewline
\cline{2-11} \cline{3-11} \cline{4-11} \cline{5-11} \cline{6-11} \cline{7-11} \cline{8-11} \cline{9-11} \cline{10-11} \cline{11-11} 
 & {\footnotesize{}Train + Test} & {\footnotesize{}66.9$\pm$0.8} & \textbf{\footnotesize{}56.9$\pm$0.7} & {\footnotesize{}47.5$\pm$0.5} & {\footnotesize{}37.7$\pm$0.4} & {\footnotesize{}35.7$\pm$0.5} & {\footnotesize{}21.6$\pm$0.3} & \textbf{\footnotesize{}61.2$\pm$1.2} & \textbf{\footnotesize{}52.7$\pm$0.8} & \textbf{\footnotesize{}43.4$\pm$1.3}\tabularnewline
\hline 
\hline 
\multirow{2}{*}{{\footnotesize{}CRLC}} & {\footnotesize{}Train only} & {\footnotesize{}79.4$\pm$0.3} & {\footnotesize{}66.7$\pm$0.6} & {\footnotesize{}62.3$\pm$0.4} & \textbf{\footnotesize{}43.4$\pm$0.8} & \textbf{\footnotesize{}43.1$\pm$0.5} & \textbf{\footnotesize{}27.7$\pm$0.3} & {\footnotesize{}57.6$\pm$1.6} & {\footnotesize{}50.8$\pm$1.5} & {\footnotesize{}41.9$\pm$1.2}\tabularnewline
\cline{2-11} \cline{3-11} \cline{4-11} \cline{5-11} \cline{6-11} \cline{7-11} \cline{8-11} \cline{9-11} \cline{10-11} \cline{11-11} 
 & {\footnotesize{}Train + Test} & \textbf{\footnotesize{}79.9$\pm$0.6} & \textbf{\footnotesize{}67.9$\pm$0.6} & \textbf{\footnotesize{}63.4$\pm$0.4} & {\footnotesize{}42.5$\pm$0.7} & {\footnotesize{}41.6$\pm$0.8} & {\footnotesize{}26.3$\pm$0.5} & \textbf{\footnotesize{}81.8$\pm$0.3} & \textbf{\footnotesize{}72.9$\pm$0.4} & \textbf{\footnotesize{}68.2$\pm$0.3}\tabularnewline
\hline 
\end{tabular}{\footnotesize\par}
\par\end{centering}
\caption{Clustering results of our proposed methods on CIFAR10, CIFAR20 and
STL10 with only the training set used and with both the training and
test sets used.\label{tab:Complete-results-on-CIFAR-n-STL10}}
\end{table*}

\begin{table*}
\begin{centering}
{\footnotesize{}}%
\begin{tabular}{|c|ccc|ccc|}
\hline 
\multicolumn{1}{|c|}{{\footnotesize{}Dataset}} & \multicolumn{3}{c|}{{\footnotesize{}ImageNet10}} & \multicolumn{3}{c|}{{\footnotesize{}ImageNet-Dogs}}\tabularnewline
\hline 
\multicolumn{1}{|c|}{{\footnotesize{}Metric}} & {\footnotesize{}ACC} & {\footnotesize{}NMI} & {\footnotesize{}ARI} & {\footnotesize{}ACC} & {\footnotesize{}NMI} & {\footnotesize{}ARI}\tabularnewline
\hline 
\hline 
\multirow{1}{*}{{\footnotesize{}C-head only}} & {\footnotesize{}80.0$\pm$1.4} & {\footnotesize{}75.2$\pm$1.9} & {\footnotesize{}67.6$\pm$2.2} & {\footnotesize{}36.3$\pm$0.9} & {\footnotesize{}37.5$\pm$0.7} & {\footnotesize{}19.8$\pm$0.4}\tabularnewline
\hline 
\multirow{1}{*}{{\footnotesize{}CRLC}} & \textbf{\footnotesize{}85.4$\pm$0.3} & \textbf{\footnotesize{}83.1$\pm$0.5} & \textbf{\footnotesize{}75.9$\pm$0.4} & \textbf{\footnotesize{}46.1$\pm$0.6} & \textbf{\footnotesize{}48.4$\pm$0.6} & \textbf{\footnotesize{}29.7$\pm$0.4}\tabularnewline
\hline 
\end{tabular}{\footnotesize\par}
\par\end{centering}
\caption{Clustering results of our proposed methods on ImageNet10 and ImageNet-Dogs.\label{tab:Complete-results-on-imagent10_n_dogs}}
\end{table*}

\begin{figure}
\centering{}\includegraphics[width=0.8\columnwidth]{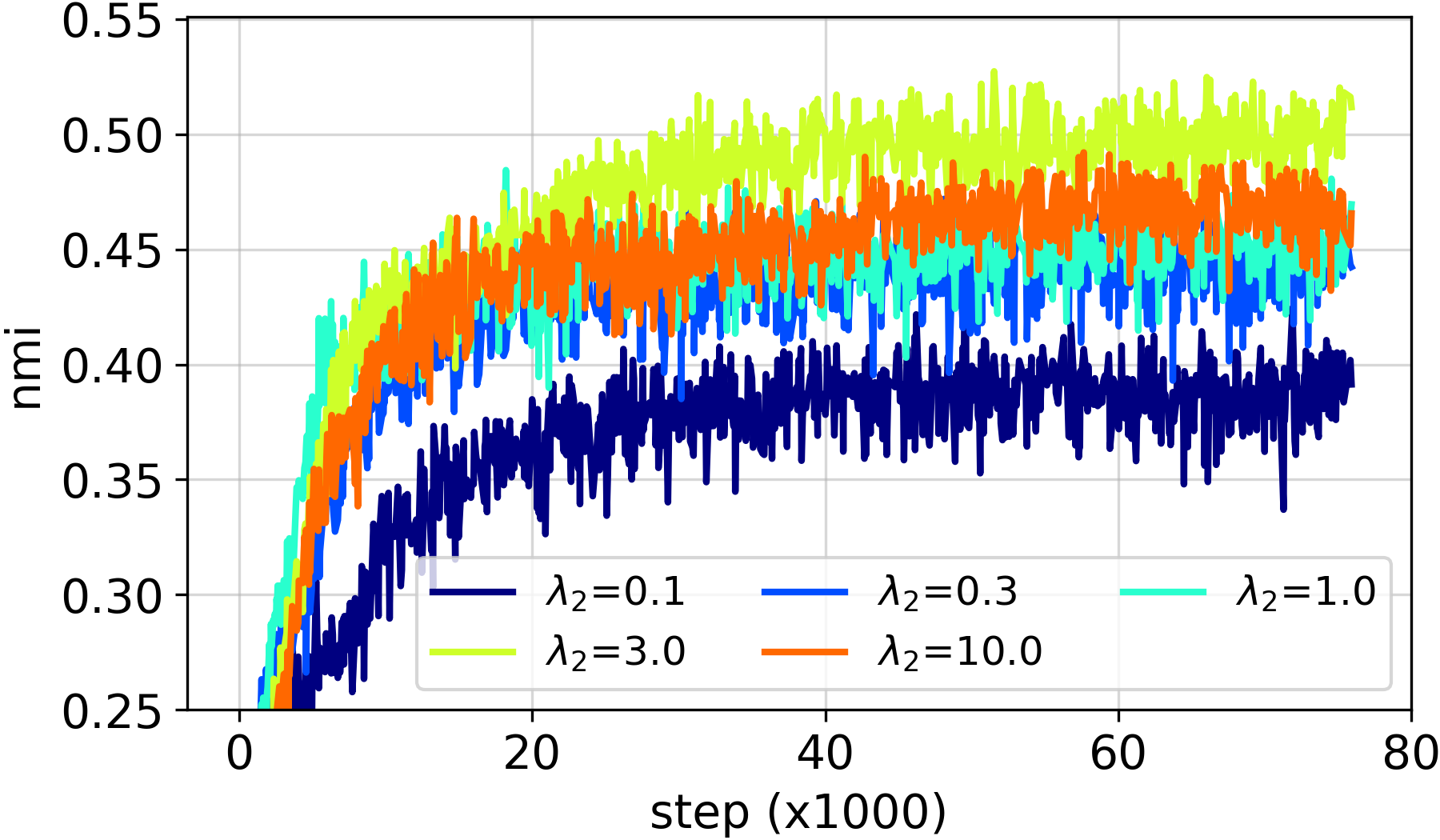}\caption{NMI curve of CRLC on ImageNet-Dogs w.r.t. different coefficients of
$\protect\Loss_{\text{FC}}$.\label{fig:LossFC-contrib-ImageNetDogs}}
\end{figure}

\subsection{Dataset description\label{subsec:Dataset-description}}

In Table~\ref{tab:Datasets_Details}, we provide details of the datasets
used in this work. CIFAR20 is CIFAR100 with 100 classes replaced by
20 super-classes. STL10 is different from other datasets in the sense
that it has an auxiliary set of 100,000 unlabeled samples of unknown
classes. Similar to previous works, we use samples from this auxiliary
set and the training set to train the ``representation learning''
head.

\subsection{Training setups for clustering\label{subsec:Training-setups-for-clustering}}

\paragraph{End-to-end clustering}

For end-to-end clustering, we use a SGD optimizer with a constant
learning rate = 0.1, momentum = 0.9, Nesterov = False, and weight
decay = 5e-4 based on the settings in \cite{goyal2017nonparametric,he2020momentum,tian2019contrastive}.
We set the batch size to 512 and the number of epochs to 2000. In
fact, on some datasets like ImageNet10 or ImageNet-Dogs, CRLC only
needs 500 epochs to converge. The coefficients of the negative entropy
and $\Loss_{\text{FC}}$ ($\lambda_{1}$ and $\lambda_{2}$ in Eq.~11
in the main text) are fixed at 1 and 10, respectively. Each experiment
is repeated 3 times with random initializations.

\paragraph{Two-stage clustering}

For two-stage clustering, we use the same settings as in \cite{van2020scan}.
Specifically, the backbone network is ResNet18 for CIFAR10/20, STL10
and is ResNet50 for ImageNet50/100/200. In the first (pretraining)
stage, for CIFAR10/20 and STL10, we pretrain the backbone network
and the RL-head via SimCLR \cite{chen2020simple} for 500 epochs.
The optimizer is SGD with an initial learning rate = 0.4 decayed with
a cosine decay schedule \cite{loshchilov2016sgdr}, momentum = 0.9,
Nesterov = False, and weight decay = 1e-4. Meanwhile, for ImageNet50/100/200,
we directly copy the pretrained weights of MoCo \cite{he2020momentum}
to the backbone network and the RL-head. After the pretraining stage,
we find for each sample in the training set 50 nearest neighbors based
on the cosine similarity measure. Positive samples for contrative
learning in the second stage are drawn uniformly from these sets of
nearest neighbors. In the second stage, for CIFAR10/20 and STL10,
we train both the backbone network and the C-head for 200 epochs by
minimizing $\Loss_{\text{cluster}}$ (Eq.~8 in the main text) using
an Adam optimizer with a constant learning rate = 1e-4 and weight
decay = 1e-4. For ImageNet50/100/200, we freeze the backbone network
and only train the C-head for 200 epochs by minimizing $\Loss_{\text{cluster}}$
using an SGD optimizer with a constant learning rate = 5.0, momentum
= 0.9, Nesterov = False, and weight decay = 0.0.

\begin{table*}
\begin{centering}
{\footnotesize{}}%
\begin{tabular}{|c|ccc|ccc|ccc|}
\hline 
\multicolumn{1}{|c|}{{\footnotesize{}Dataset}} & \multicolumn{3}{c|}{{\footnotesize{}CIFAR10}} & \multicolumn{3}{c|}{{\footnotesize{}CIFAR20}} & \multicolumn{3}{c|}{{\footnotesize{}STL10}}\tabularnewline
\hline 
\multicolumn{1}{|c|}{{\footnotesize{}Metric}} & {\footnotesize{}ACC} & {\footnotesize{}NMI} & {\footnotesize{}ARI} & {\footnotesize{}ACC} & {\footnotesize{}NMI} & {\footnotesize{}ARI} & {\footnotesize{}ACC} & {\footnotesize{}NMI} & {\footnotesize{}ARI}\tabularnewline
\hline 
\hline 
{\footnotesize{}K-means \cite{van2020scan}} & {\footnotesize{}65.9$\pm$5.7} & {\footnotesize{}59.8$\pm$2.0} & {\footnotesize{}50.9$\pm$3.7} & {\footnotesize{}39.5$\pm$1.9} & {\footnotesize{}40.2$\pm$1.1} & {\footnotesize{}23.9$\pm$1.1} & {\footnotesize{}65.8$\pm$5.1} & {\footnotesize{}60.4$\pm$2.5} & {\footnotesize{}50.6$\pm$4.1}\tabularnewline
\hline 
{\footnotesize{}SCAN \cite{van2020scan}} & {\footnotesize{}81.8$\pm$0.3} & {\footnotesize{}71.2$\pm$0.4} & {\footnotesize{}66.5$\pm$0.4} & {\footnotesize{}42.2$\pm$3.0} & {\footnotesize{}44.1$\pm$1.0} & {\footnotesize{}26.7$\pm$1.3} & {\footnotesize{}75.5$\pm$2.0} & {\footnotesize{}65.4$\pm$1.2} & {\footnotesize{}59.0$\pm$1.6}\tabularnewline
\hline 
\hline 
{\footnotesize{}two-stage CRLC} & \textbf{\footnotesize{}84.2$\pm$0.1} & \textbf{\footnotesize{}74.7$\pm$0.3} & \textbf{\footnotesize{}70.6$\pm$0.5} & \textbf{\footnotesize{}45.0$\pm$0.7} & \textbf{\footnotesize{}44.8$\pm$0.8} & \textbf{\footnotesize{}28.7$\pm$0.9} & \textbf{\footnotesize{}78.7$\pm$1.1} & \textbf{\footnotesize{}68.4$\pm$1.6} & \textbf{\footnotesize{}62.7$\pm$1.8}\tabularnewline
\hline 
\end{tabular}{\footnotesize\par}
\par\end{centering}
\caption{Two-stage clustering results on CIFAR10/20 and STL10.\label{tab:Two-stage-results-CIFAR}}
\end{table*}

\begin{figure*}
\begin{centering}
\subfloat[NMI\label{fig:mem_bank_num_neg_CIFAR20_a}]{\begin{centering}
\includegraphics[width=0.32\textwidth]{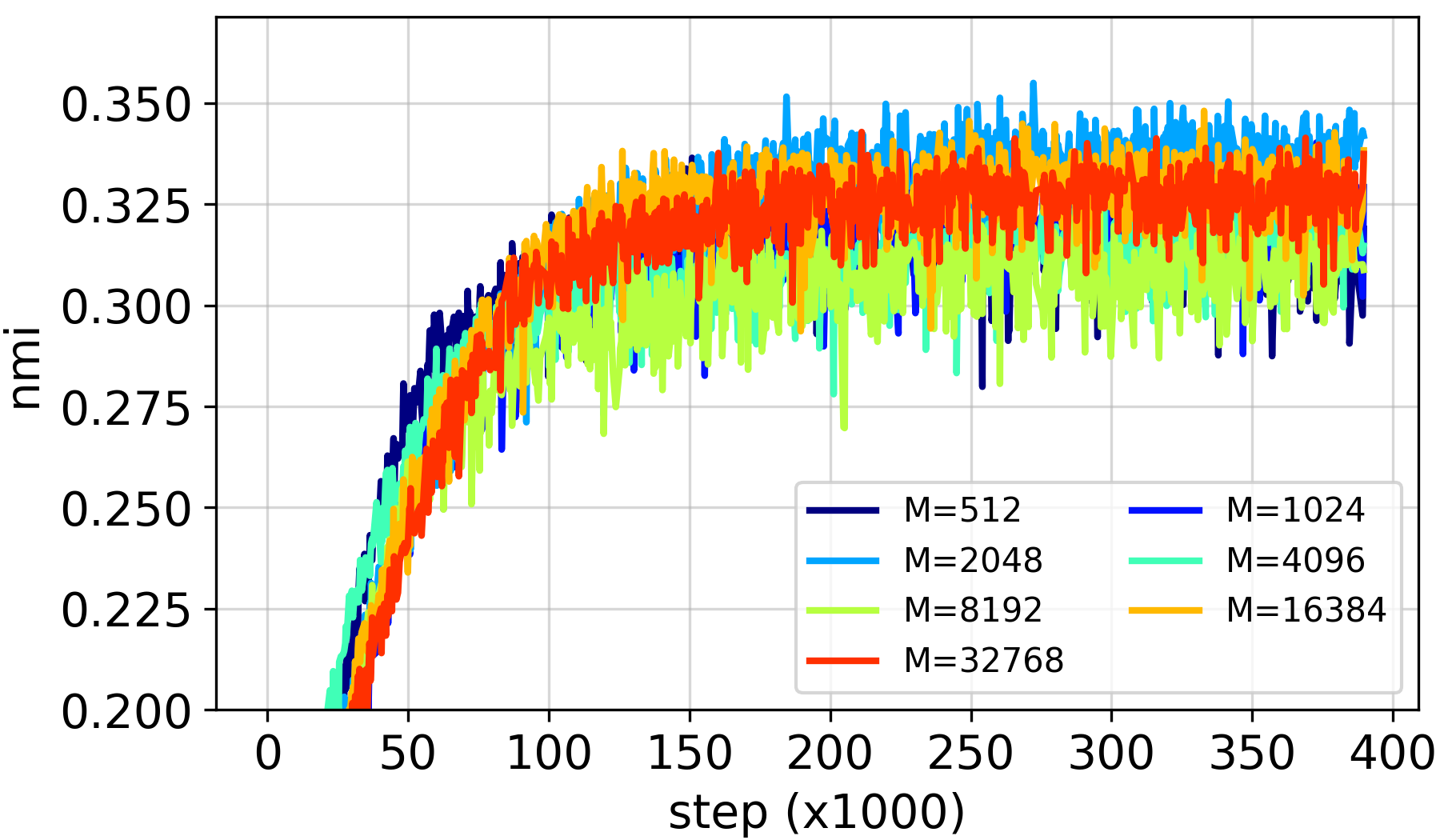}
\par\end{centering}
}\subfloat[InfoNCE w.r.t. $\protect\Loss_{\text{FC}}$\label{fig:mem_bank_num_neg_CIFAR20_b}]{\begin{centering}
\includegraphics[width=0.32\textwidth]{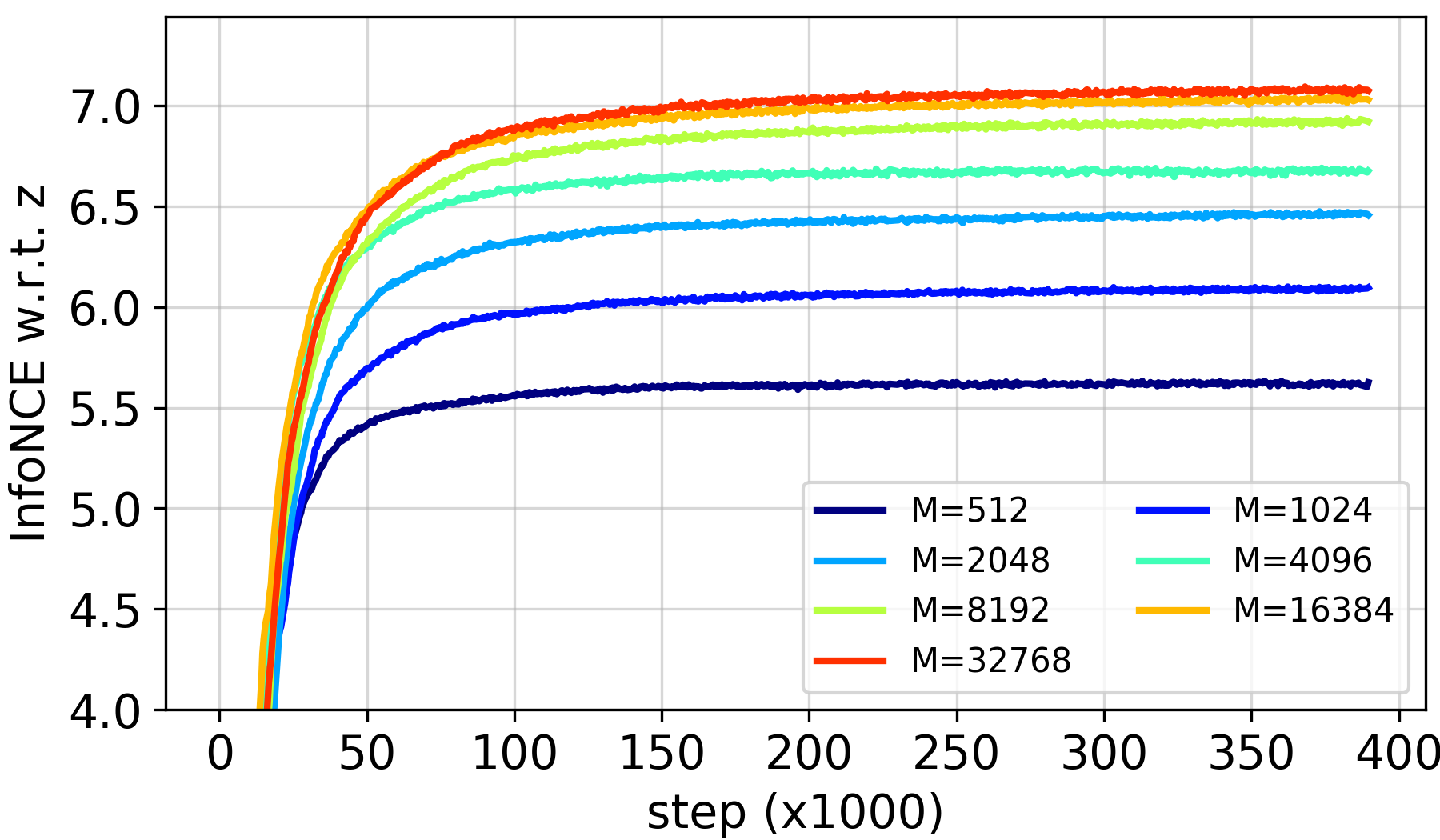}
\par\end{centering}
}\subfloat[InfoNCE w.r.t. $\protect\Loss_{\text{PC}}$\label{fig:mem_bank_num_neg_CIFAR20_c}]{\begin{centering}
\includegraphics[width=0.32\textwidth]{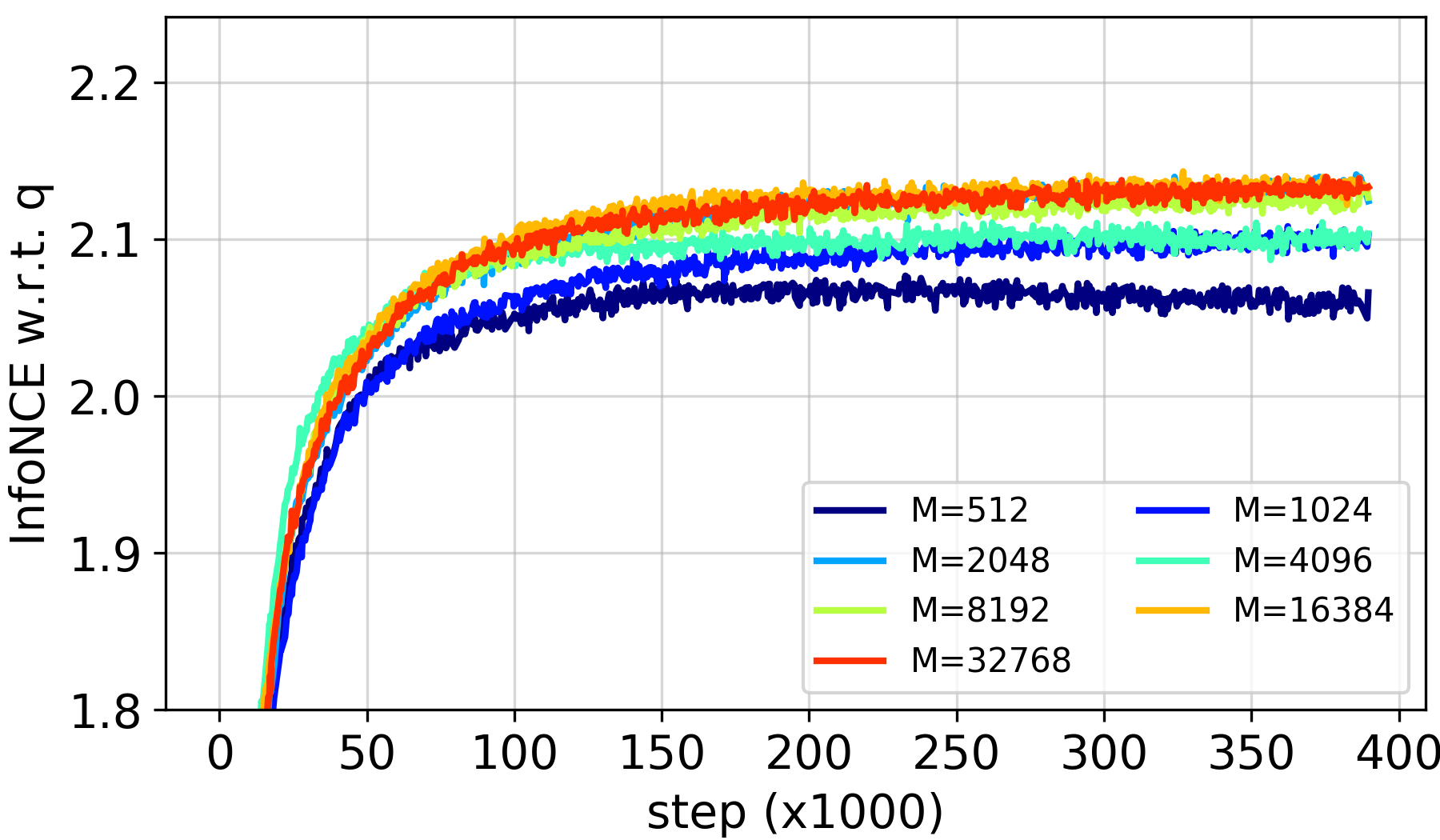}
\par\end{centering}
}
\par\end{centering}
\caption{Learning curves of MemoryBank-based CRLC on CIFAR20 w.r.t. different
numbers of negative samples. The momentum is $\alpha=0.5$. The InfoNCE
w.r.t. a contrastive loss is computed by using Eq.~2 in the main
text.\label{fig:mem_bank_num_neg_CIFAR20}}
\end{figure*}

\begin{figure*}
\begin{centering}
\subfloat[NMI\label{fig:mem_bank_momentum_CIFAR20_a}]{\begin{centering}
\includegraphics[width=0.32\textwidth]{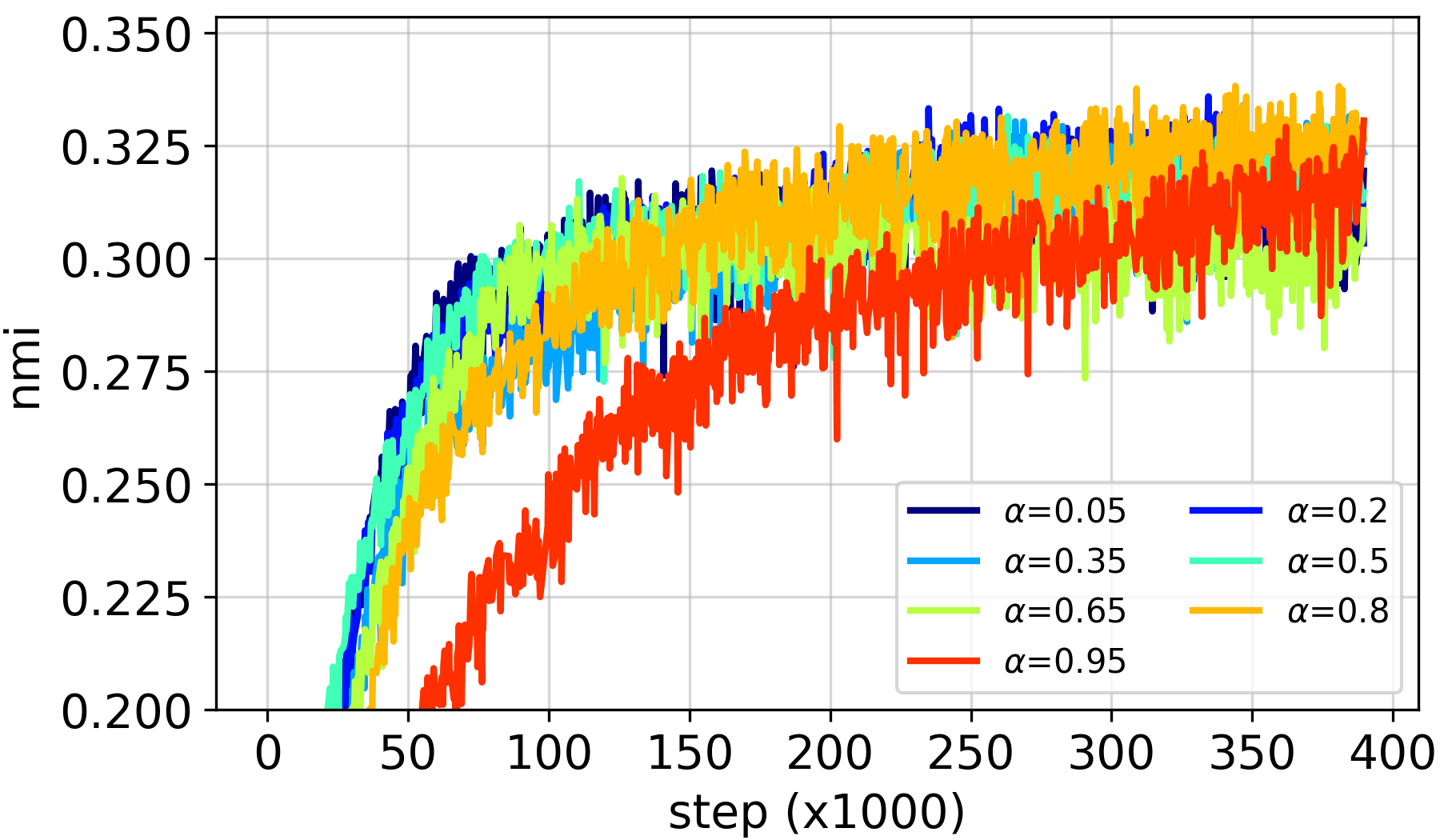}
\par\end{centering}
}\subfloat[InfoNCE w.r.t. $\protect\Loss_{\text{FC}}$\label{fig:mem_bank_momentum_CIFAR20_b}]{\begin{centering}
\includegraphics[width=0.32\textwidth]{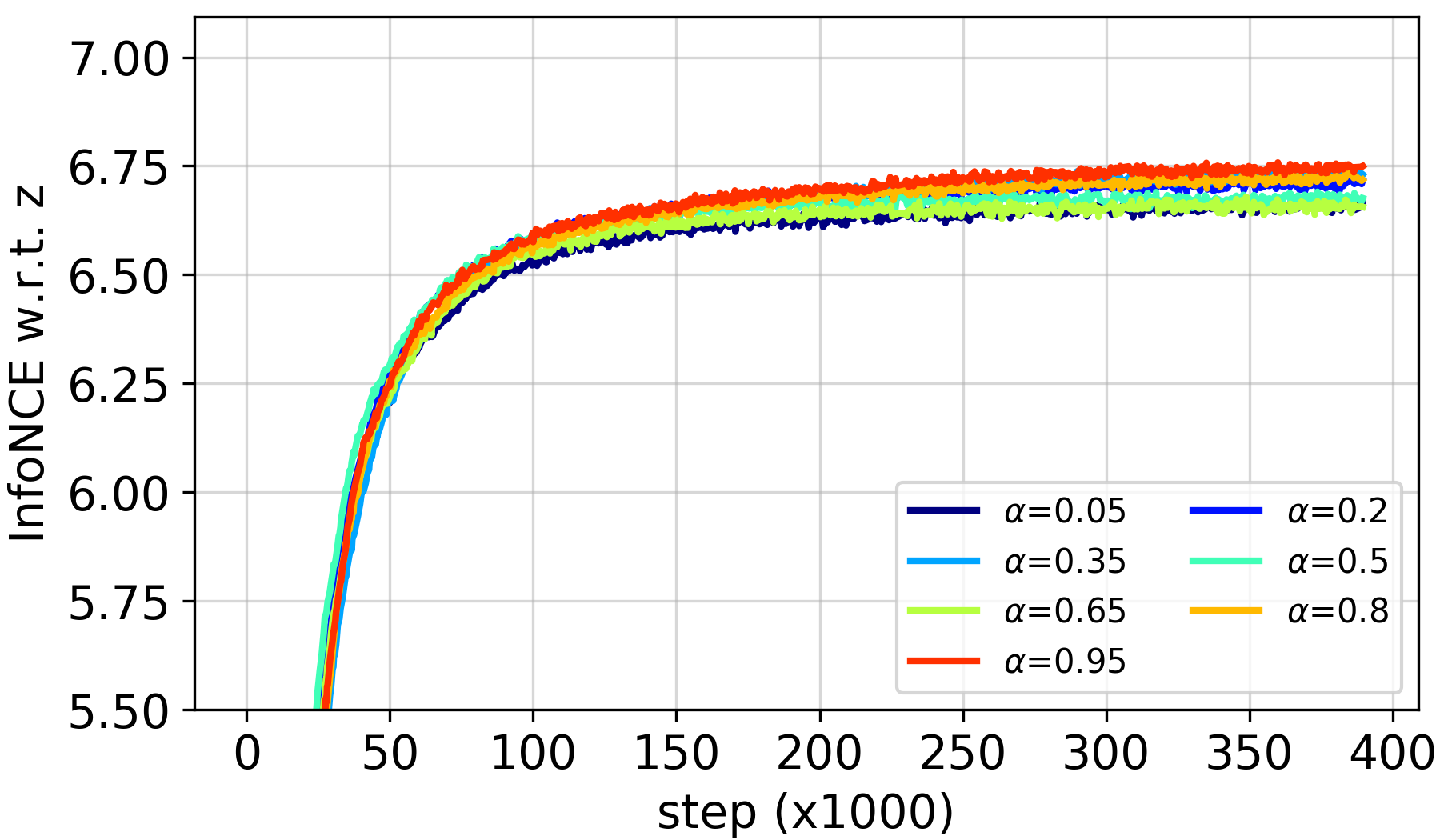}
\par\end{centering}
}\subfloat[InfoNCE w.r.t. $\protect\Loss_{\text{PC}}$\label{fig:mem_bank_momentum_CIFAR20_c}]{\begin{centering}
\includegraphics[width=0.32\textwidth]{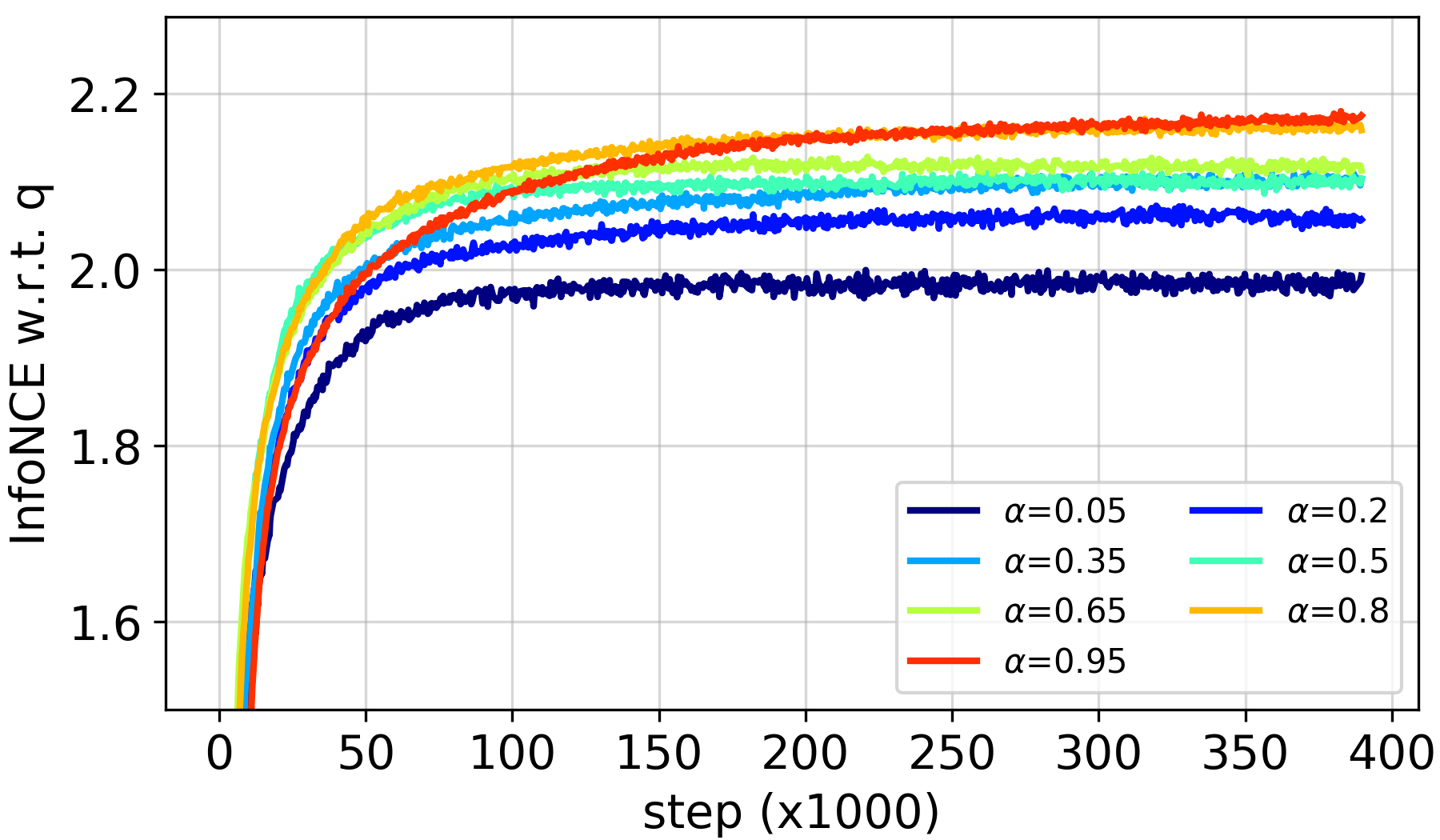}
\par\end{centering}
}
\par\end{centering}
\caption{Learning curves of MemoryBank-based CRLC on CIFAR20 w.r.t. different
values of the momentum. The number of negative samples is $M=4096$.
The InfoNCE w.r.t. a contrastive loss is computed by using Eq.~2
in the main text.\label{fig:mem_bank_momentum_CIFAR20}}
\end{figure*}

\subsection{Complete end-to-end clustering results\label{subsec:Complete-N2N-results}}

Complete results with standard deviations on the five standard clustering
datasets are shown in Tables~\ref{tab:Complete-results-on-CIFAR-n-STL10}
and \ref{tab:Complete-results-on-imagent10_n_dogs}. From Table~\ref{tab:Complete-results-on-CIFAR-n-STL10},
we see that for CIFAR10 using both the training and test sets does
not cause much difference in performance compared to using only the
training set. For CIFAR20, using only the training set even leads
to slightly better results. By contrast, for STL10, models trained
with both the training and test sets significantly outperform those
trained with the training set only. We believe the reason is that
for CIFAR10 and CIFAR20, the training set is big enough to cover the
data distribution in the test set while for STL10, it does not apply
(Table~\ref{tab:Datasets_Details}). Therefore, we think subsequent
works should use only the training set when doing experiments on CIFAR10
and CIFAR20.

\subsection{Additional two-stage clustering results\label{subsec:Additional-two-stage-clustering-results}}

Table~\ref{tab:Two-stage-results-CIFAR} compares the clustering
results of ``two-stage'' CRLC and SCAN on CIFAR10/20, STL10. ``Two-stage''
CRLC clearly outperforms SCAN on all datasets.

\subsection{Additional ablation study results}

\subsubsection{Contribution of the feature contrastive loss\label{subsec:Contribution-of-the-LossFC-Appdx}}

In Fig.~\ref{fig:LossFC-contrib-ImageNetDogs}, we show the performance
of CRLC on ImageNet-Dogs w.r.t. different coefficients of $\Loss_{\text{FC}}$
($\lambda_{2}$ in Eq.~11 in the main text). We observe that CRLC
achieves the best clustering accuracy when $\lambda_{2}=3$. However,
in Table~1 in the main text, we still report the result when $\lambda_{2}=10$.

\subsubsection{Nonparametric implementation of CRLC\label{subsec:Nonparametric-implementation-Appdx}}

In this section, we empirically investigate the contributions of the
number of negative samples and the momentum coefficient ($\alpha$
in Eq.~10 in the main text) to the performance of MemoryBank-based
CRLC.

\paragraph{Contribution of the number of negative samples}

From Fig.~\ref{fig:mem_bank_num_neg_CIFAR20_a}, we do not see any
correlation between the number of negative samples and the clustering
performance of MemoryBank-baed CRLC despite the fact that increasing
the number of negative samples allows the RL-head and the C-head to
gain more information from data (Figs.~\ref{fig:mem_bank_num_neg_CIFAR20_b}
and \ref{fig:mem_bank_num_neg_CIFAR20_c}). It suggests that for clustering
(and possibly other classification tasks), getting more information
may not lead to good results. Instead, we need to extract the right
information related to clusters.

\paragraph{Contribution of the momentum coefficient}

From Fig.~\ref{fig:mem_bank_momentum_CIFAR20_b}, we see that changing
the momentum value for updating probability vectors stored in the
memory bank does not affects amount of information captured by the
RL-head much. By contrast, in Fig.~\ref{fig:mem_bank_momentum_CIFAR20_c},
we see that larger values of the momentum cause the C-head to capture
more information. This is reasonable because the accumulated probability
vector $q_{n,t}$ is usually more stochastic (contains more information)
than the probability vector $\hat{q}_{n}$ of a particular view (Eq.~10
in the main text). Larger values of the momentum also cause the model
to converge slower but do not affect the performance much (Fig.~\ref{fig:mem_bank_momentum_CIFAR20_a}). 

\begin{figure*}
\begin{centering}
\includegraphics[width=0.98\textwidth]{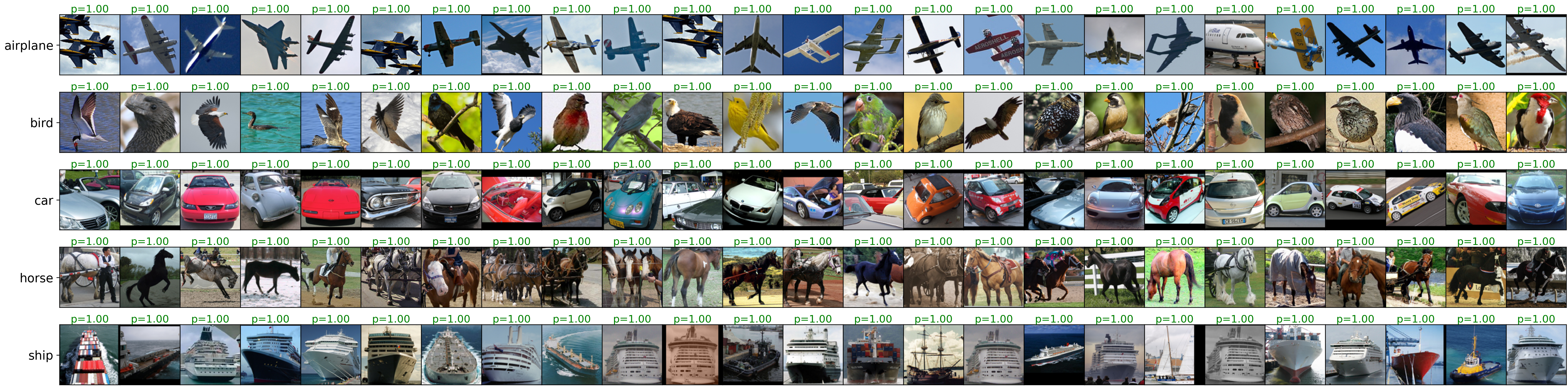}
\par\end{centering}
\caption{STL10 samples of 5 classes correctly predicted by CRLC. Samples are
sorted by their confidence scores.\label{fig:STL10-25-correct-samples}}
\end{figure*}

\subsection{Qualitative evaluation\label{subsec:Qualitative-evaluation}}

In Fig.~\ref{fig:STL10-25-correct-samples}, we show the top correctly
predicted samples according to their confidence score for each of
5 classes from the training set of STL10. It is clear that these samples
are representative of the cluster they belong to.

\subsection{Consistency-regularization-based semi-supervised learning methods\label{subsec:Consistency-regularization-based-methods}}

When some labeled data are given, the clustering problem naturally
becomes semi-supervised learning (SSL). The core idea behind recent
state-of-the-art SSL methods such as UDA \cite{xie2019unsupervised},
MixMatch \cite{berthelot2019mixmatch}, ReMixMatch \cite{berthelot2019remixmatch},
FixMatch \cite{sohn2020fixmatch} is \emph{consistency regularization}
(CR) which is about forcing an input sample under different perturbations/augmentations
to have similar class predictions. In this sense, CR can be seen as
an unnormalized version of the probability contrastive loss without
the denominator. Different SSL methods extend CR in different ways.
For example, UDA uses strong data augmentation to generate positive
pairs. MixMatch and ReMixMatch combines CR with MixUp \cite{zhang2017mixup}.
However, none of the above methods achieve consistent performance
with extremely few labeled data (Section~5.2 in the main text). By
contrast, clustering methods like CRLC perform consistently well even
when no label is available. Thus, we believe designing a method that
enjoys the strength of both fields is possible and CRLC-semi can be
one step towards that goal. 

\subsection{Training setups for semi-supervised learning\label{subsec:Training-setups-for-SSL}}

To train CRLC-semi, we use a SGD optimizer with an initial learning
rate = 0.1, momentum = 0.9, Nesterov = False, and weight decay = 5e-4.
Similar to \cite{sohn2020fixmatch}, we adjust the learning rate at
each epoch using a cosine decay schedule \cite{loshchilov2016sgdr}
computed as follows:
\[
\text{lr}_{t}=\text{\text{lr}}_{\text{min}}+(\text{lr}_{\text{init}}-\text{lr}_{\text{min}})\times\frac{1+\cos\left(\frac{t}{T}\pi\right)}{2}
\]
where $\text{lr}_{\text{init}}=0.1$, $\text{lr}_{\text{min}}=0.001$,
$\text{lr}_{t}$ is the learning rate at epoch $t$ over $T$ epochs
in total. $T$ is 2000 and 1000 for CIFAR10 and CIFAR100, respectively.
The number of labeled and unlabeled samples in each batch is 64 and
512, respectively. In $\Loss_{\text{CRLC-semi}}$ (Eq.~12 in the
main text), $\lambda_{1}=1$, $\lambda_{2}=5$, and $\lambda_{3}=1$.

We reimplement FixMatch using sample code from Github\footnote{\url{https://github.com/CoinCheung/fixmatch-pytorch}}
with the default settings unchanged. In this code, the number of labeled
and unlabeled data in a batch is 64 and 448, respectively. However,
the number of steps in one epoch does not depend on the batch size
but is fixed at 1024. Thus, FixMatch is trained in 1024 epochs $\approx$
1 million steps for both CIFAR10 and CIFAR100. Meanwhile, CLRC-semi
is trained in only 194,000 steps for CIFAR10 and 97,000 steps for
CIFAR100.

\begin{table*}
\begin{centering}
{\footnotesize{}}%
\begin{tabular}{|c|c|c|c|c|c|c|c|c|}
\hline 
{\footnotesize{}Dataset} & \multicolumn{4}{c|}{{\footnotesize{}CIFAR10}} & \multicolumn{4}{c|}{{\footnotesize{}CIFAR100}}\tabularnewline
\hline 
{\footnotesize{}Labels} & {\footnotesize{}10} & {\footnotesize{}20} & {\footnotesize{}40} & {\footnotesize{}250} & {\footnotesize{}100} & {\footnotesize{}200} & {\footnotesize{}400} & {\footnotesize{}2500}\tabularnewline
\hline 
\hline 
{\footnotesize{}$\Pi$-model \cite{laine2016temporal}} & {\footnotesize{}-} & {\footnotesize{}-} & {\footnotesize{}-} & {\footnotesize{}54.26$\pm$3.97} & {\footnotesize{}-} & {\footnotesize{}-} & {\footnotesize{}-} & {\footnotesize{}57.25$\pm$0.48}\tabularnewline
\hline 
{\footnotesize{}Pseudo Labeling \cite{lee2013pseudo}} & {\footnotesize{}-} & {\footnotesize{}-} & {\footnotesize{}-} & {\footnotesize{}49.78$\pm$0.43} & {\footnotesize{}-} & {\footnotesize{}-} & {\footnotesize{}-} & {\footnotesize{}57.38$\pm$0.46}\tabularnewline
\hline 
{\footnotesize{}Mean Teacher \cite{tarvainen2017mean}} & {\footnotesize{}-} & {\footnotesize{}-} & {\footnotesize{}-} & {\footnotesize{}32.32$\pm$2.30} & {\footnotesize{}-} & {\footnotesize{}-} & {\footnotesize{}-} & {\footnotesize{}53.91$\pm$0.57}\tabularnewline
\hline 
{\footnotesize{}MixMatch \cite{berthelot2019mixmatch}} & {\footnotesize{}-} & {\footnotesize{}-} & {\footnotesize{}47.54$\pm$11.50} & {\footnotesize{}11.05$\pm$0.86} & {\footnotesize{}-} & {\footnotesize{}-} & {\footnotesize{}67.61$\pm$1.32} & {\footnotesize{}39.94$\pm$0.37}\tabularnewline
\hline 
{\footnotesize{}UDA \cite{xie2019unsupervised}} & {\footnotesize{}-} & {\footnotesize{}-} & {\footnotesize{}29.05$\pm$5.93} & {\footnotesize{}8.82$\pm$1.08} & {\footnotesize{}-} & {\footnotesize{}-} & {\footnotesize{}59.28$\pm$0.88} & {\footnotesize{}33.13$\pm$0.22}\tabularnewline
\hline 
{\footnotesize{}ReMixMatch \cite{berthelot2019remixmatch}} & {\footnotesize{}-} & {\footnotesize{}-} & {\footnotesize{}19.10$\pm$9.64} & {\footnotesize{}5.44$\pm$0.05} & {\footnotesize{}-} & {\footnotesize{}-} & {\footnotesize{}44.28$\pm$2.06} & {\footnotesize{}27.43$\pm$0.31}\tabularnewline
\hline 
{\footnotesize{}FixMatch (RA) \cite{sohn2020fixmatch}} & {\footnotesize{}-} & {\footnotesize{}-} & {\footnotesize{}13.81$\pm$3.37} & {\footnotesize{}5.07$\pm$0.65} & {\footnotesize{}-} & {\footnotesize{}-} & \textbf{\footnotesize{}48.85$\pm$1.75} & {\footnotesize{}28.29$\pm$0.11}\tabularnewline
\hline 
\hline 
{\footnotesize{}ReMixMatch$^{\dagger}$}\tablefootnote{{\footnotesize{}https://github.com/google-research/remixmatch}} & {\footnotesize{}59.86$\pm$9.34} & {\footnotesize{}41.68$\pm$8.15} & {\footnotesize{}28.31$\pm$6.72} & {\footnotesize{}-} & \textbf{\footnotesize{}76.32$\pm$4.30} & \textbf{\footnotesize{}66.51$\pm$2.86} & {\footnotesize{}52.23$\pm$1.71} & {\footnotesize{}-}\tabularnewline
\hline 
{\footnotesize{}FixMatch (RA)$^{\dagger}$}\tablefootnote{{\footnotesize{}https://github.com/CoinCheung/fixmatch-pytorch}} & \textbf{\footnotesize{}25.49$\pm$7.74} & \textbf{\footnotesize{}21.15$\pm$8.96} & \textbf{\footnotesize{}8.87$\pm$4.29} & {\footnotesize{}-} & {\footnotesize{}79.27$\pm$2.65} & {\footnotesize{}68.58$\pm$0.7} & {\footnotesize{}57.52$\pm$1.5} & {\footnotesize{}-}\tabularnewline
\hline 
\hline 
{\footnotesize{}CRLC-semi} & {\footnotesize{}46.75$\pm$8.01} & {\footnotesize{}29.81$\pm$1.18} & {\footnotesize{}19.87$\pm$0.82} & {\footnotesize{}13.53$\pm$0.21} & {\footnotesize{}82.20$\pm$1.15} & {\footnotesize{}73.04$\pm$1.15} & {\footnotesize{}60.87$\pm$0.17} & {\footnotesize{}41.10$\pm$0.12}\tabularnewline
\hline 
\end{tabular}{\footnotesize\par}
\par\end{centering}
\caption{Full classification errors on CIFAR10 and CIFAR100. Lower values are
better. Results of baselines are taken from \cite{sohn2020fixmatch}.
$^{\dagger}$: Results obtained from external implementations of models.\label{tab:Semi-supervised-learning-results-FULL}}
\end{table*}

\begin{figure*}
\begin{centering}
\subfloat[CIFAR10]{\begin{centering}
\includegraphics[width=0.25\textwidth]{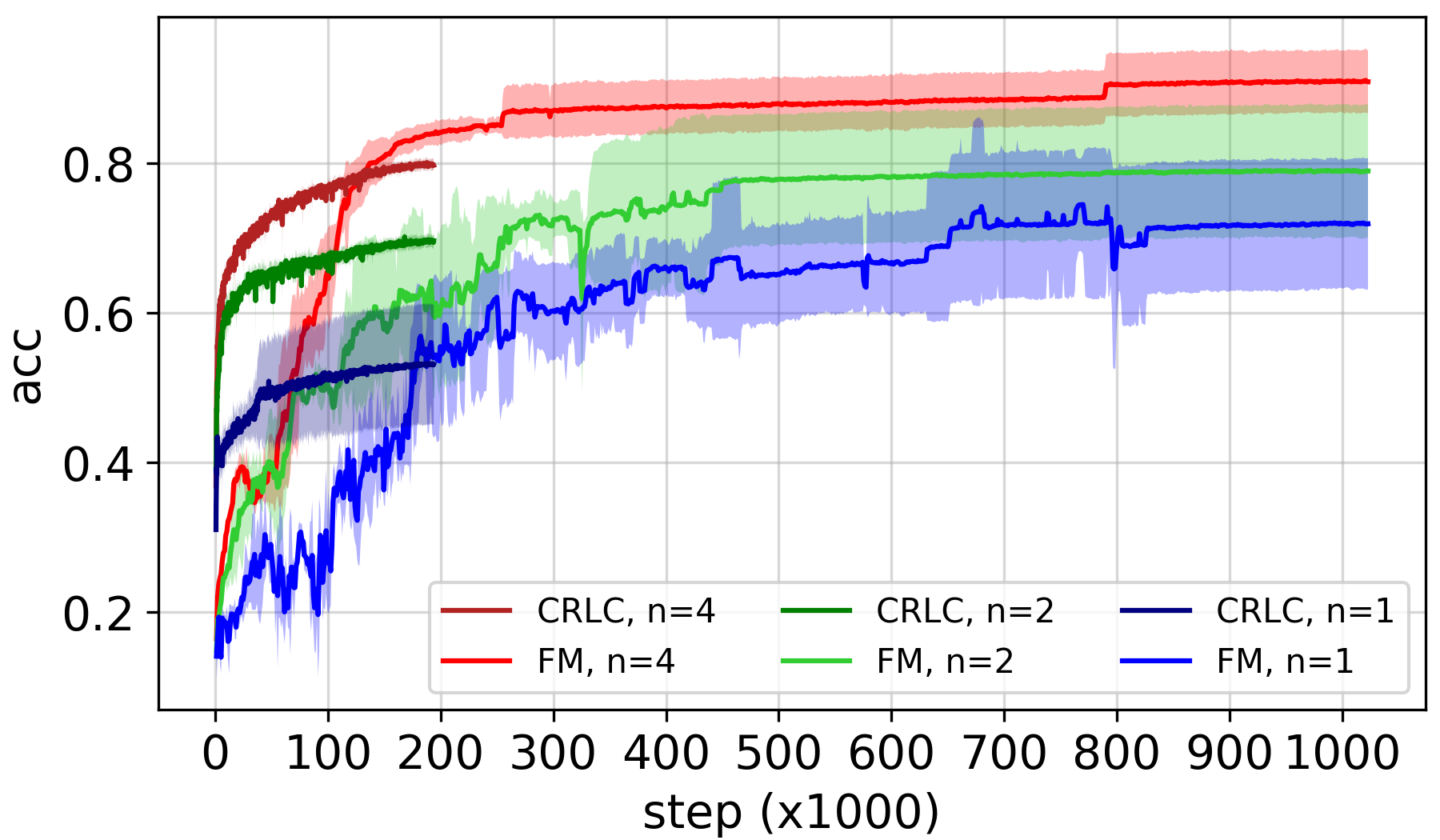}\includegraphics[width=0.25\textwidth]{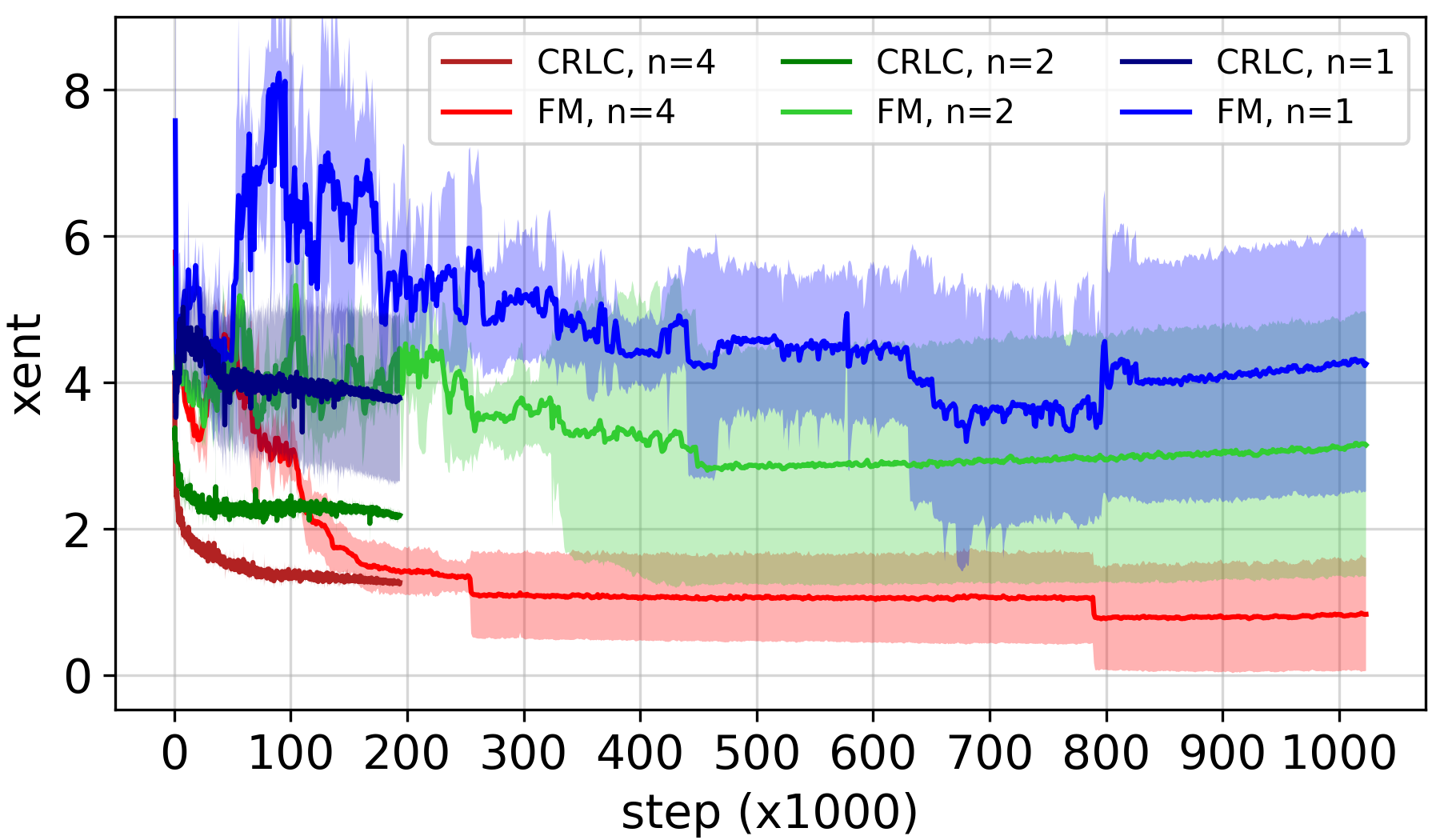}
\par\end{centering}
}\subfloat[CIFAR100]{\begin{centering}
\includegraphics[width=0.25\textwidth]{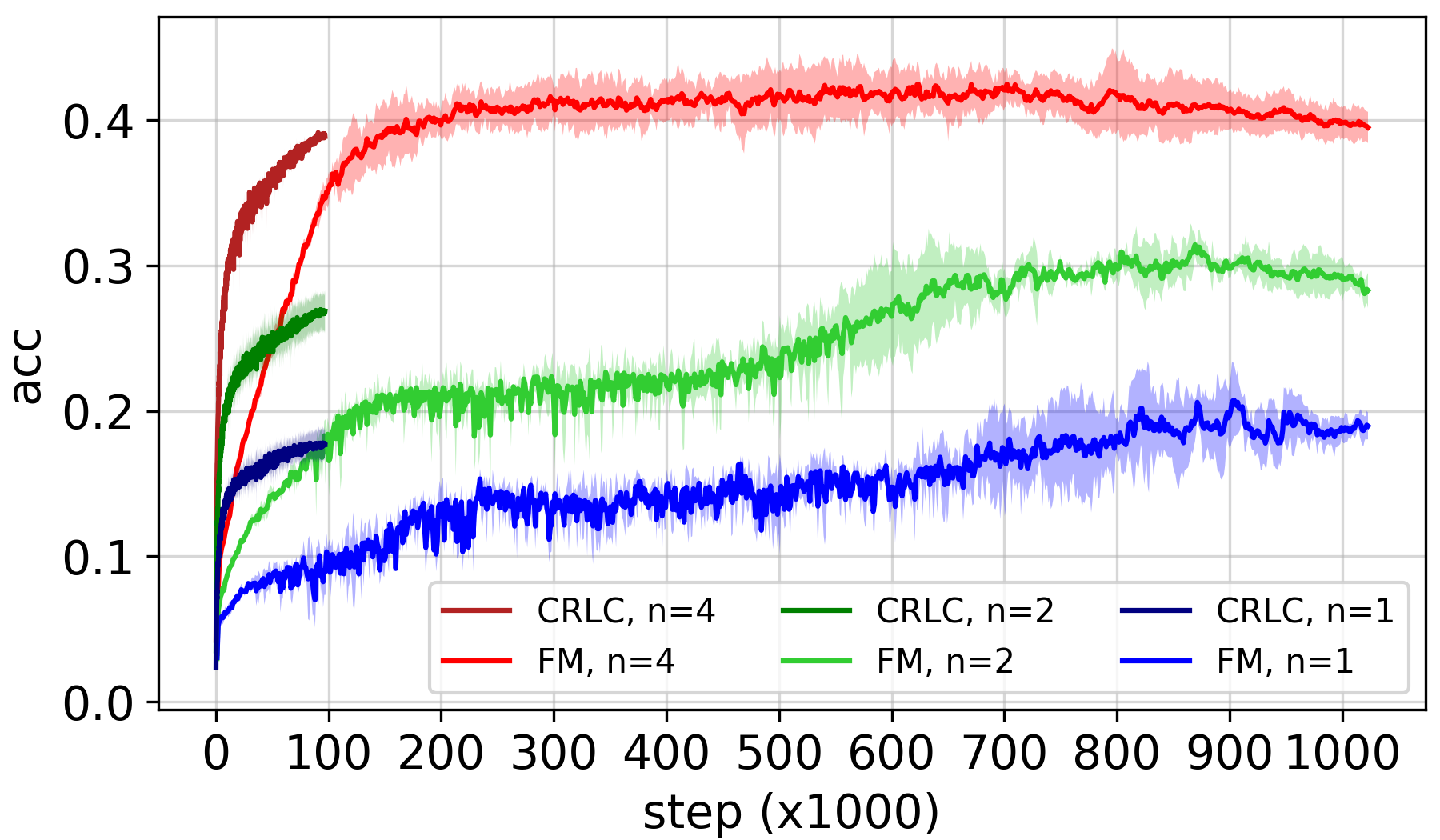}\includegraphics[width=0.25\textwidth]{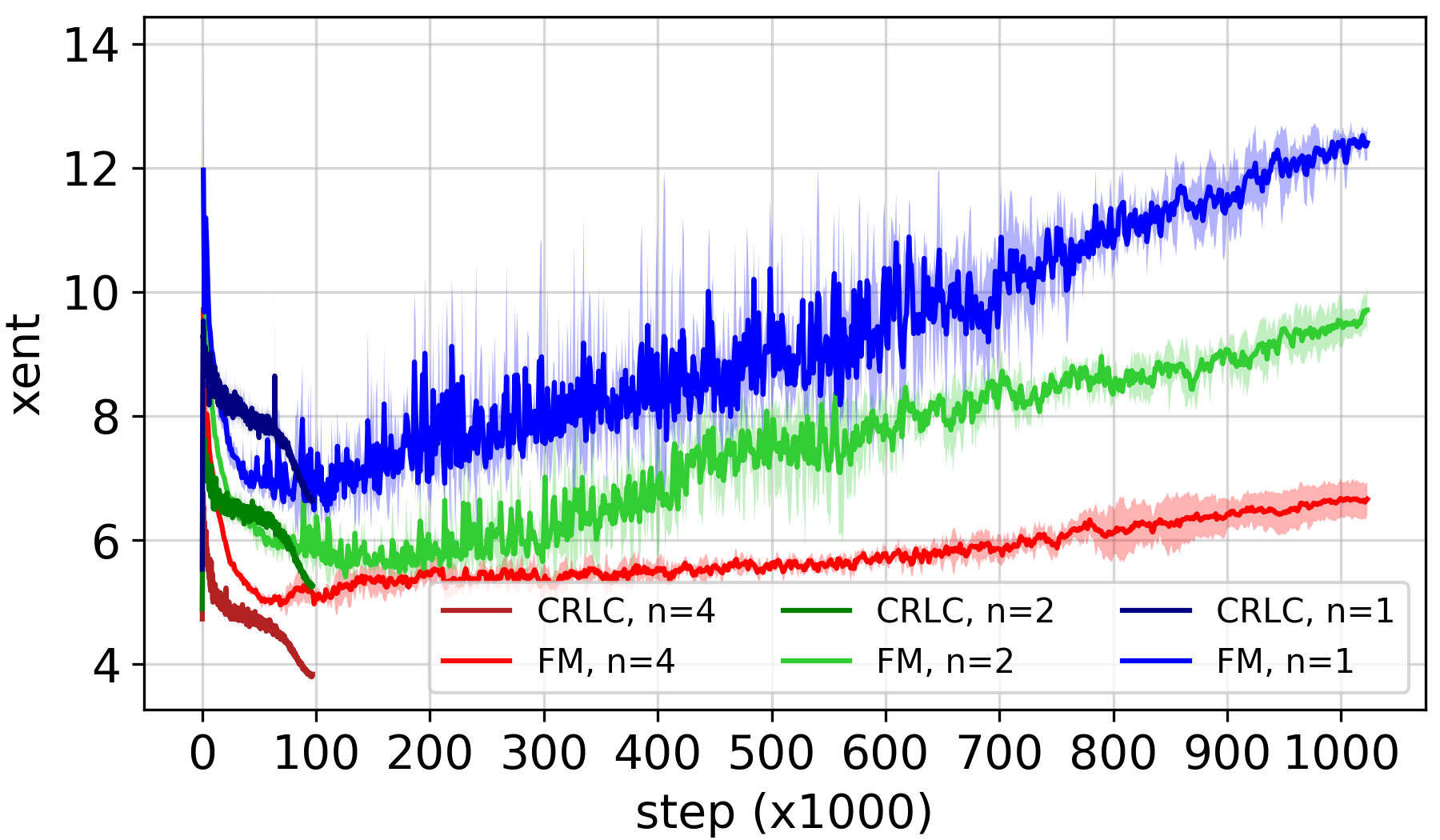}
\par\end{centering}
}
\par\end{centering}
\caption{Test accuracy and crossentropy curves of CRLC-semi (CRLC) and FixMatch
(FM) on CIFAR10 and CIFAR100 with 1, 2, 4 labeled samples per class.
It is clear that CRLC-semi performs consistently in all cases except
for the case of CIFAR10 with 1 labeled sample per class. However,
even in that case, the CRLC-semi still gives consistent performance
for each run (Fig.~\ref{fig:CIFAR10_1L_3run_curves}). FixMatch,
by contrast, is very inconsistent in its performance for each run,
especially on CIFAR10.\label{fig:Learning-curves-CRLC-FixMatch}}
\end{figure*}

\subsection{More results on semi-supervised learning\label{subsec:More-results-on-SSL}}

\begin{figure}
\begin{centering}
\includegraphics[width=0.8\columnwidth]{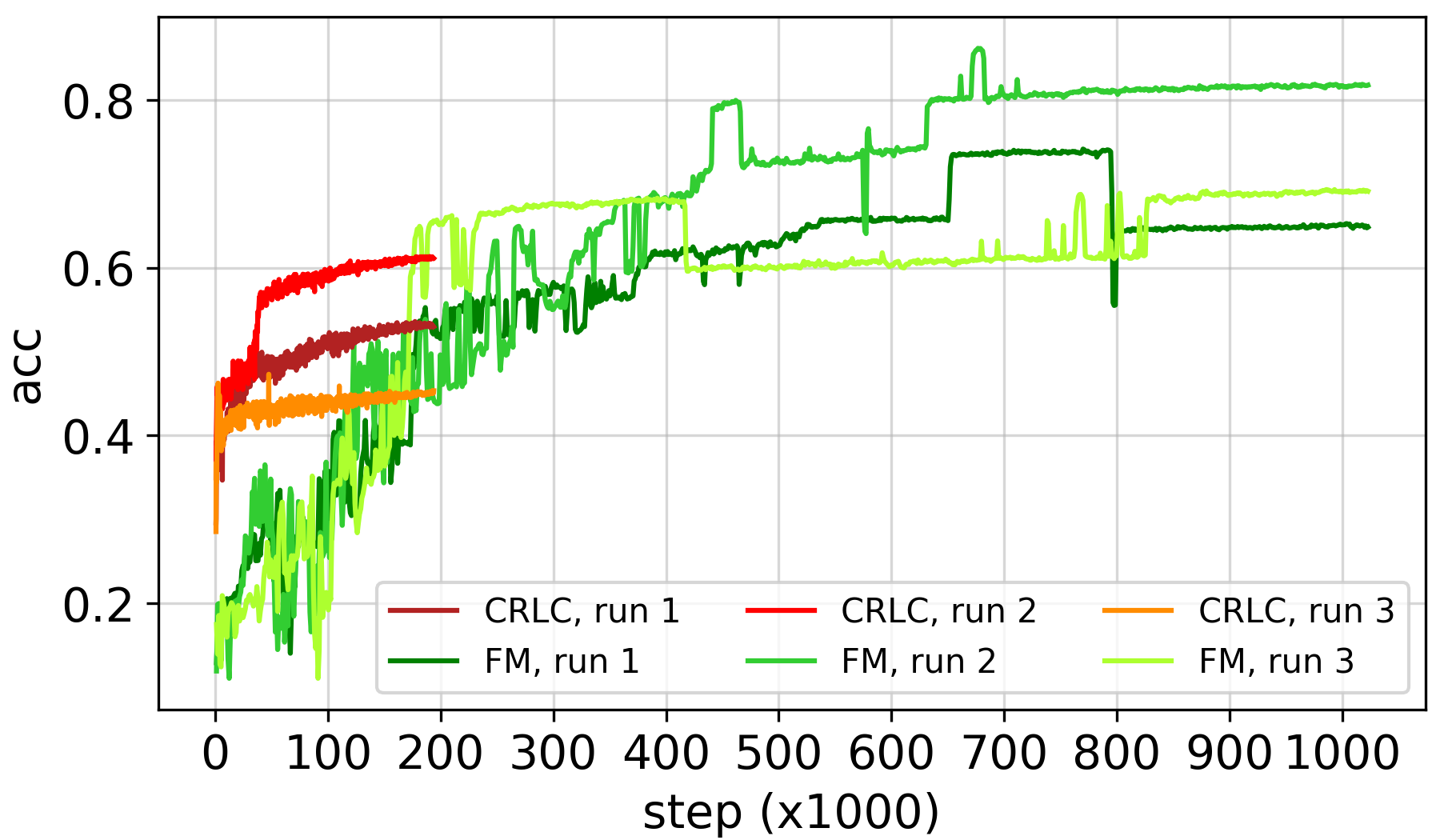}
\par\end{centering}
\caption{Test accuracy curves of CRLC-semi (CRLC) and FixMatch (FM) on CIFAR10
with 1 labeled samples per class w.r.t. 3 different runs.\label{fig:CIFAR10_1L_3run_curves}}
\end{figure}

\begin{figure}
\begin{centering}
\includegraphics[width=1\columnwidth]{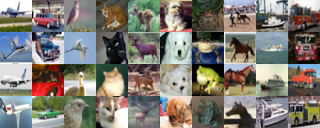}
\par\end{centering}
\caption{40 labeled CIFAR10 samples organized into 4 rows where each row has
10 images corresponding to 10 classes. For 10 and 20 labeled samples,
the first row and the first two rows are considered, respectively.\label{fig:40-labeled-CIFAR10}}
\end{figure}

\begin{figure}
\begin{centering}
\includegraphics[width=1\columnwidth]{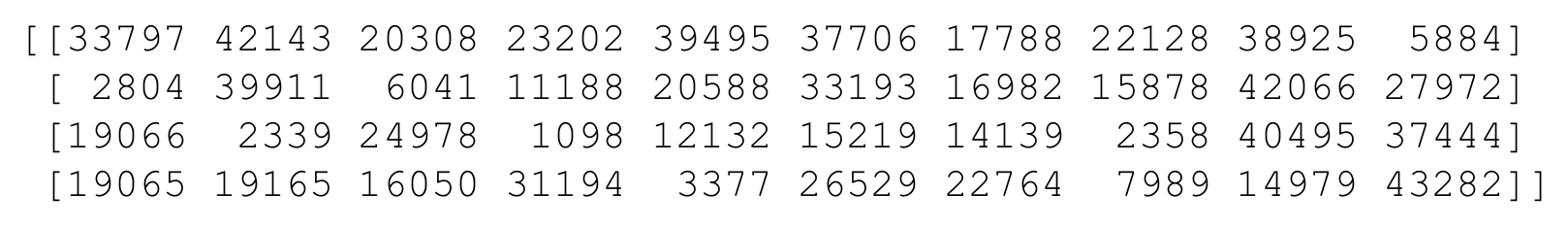}
\par\end{centering}
\caption{Indices in the training set of the images in Fig.~\ref{fig:40-labeled-CIFAR10}\label{fig:Indices-40-labeled-CIFAR10}}
\end{figure}

In Table~\ref{tab:Semi-supervised-learning-results-FULL}, we show
additional semi-supervised learning results of CRLC-semi on CIFAR10
and CIFAR100 in comparison with more baselines. CRLC-semi clearly
outperforms all standard baselines like $\Pi$-model, Pseudo Labeling
or Mean Teacher. However, CRLC-semi looses its advantage over holistic
methods like MixMatch \cite{berthelot2019mixmatch} and methods that
use strong data augmentation like UDA \cite{xie2019unsupervised}
or ReMixMatch \cite{berthelot2019remixmatch} when the number of labeled
data is big enough. Currently, we are not sure whether the problem
comes from the feature contrastive loss $\Loss_{\text{FC}}$ (when
we have enough labels, representation learning may act as a regularization
term and reduce the classification result), or from the negative entropy
term in $\Loss_{\text{cluster}}$ (causing too much regularization),
or even from the probability contrastive loss (contrasting probabilities
of two related views is not suitable when we have enough labels).
Thus, we leave the answer of this question for future work. To gain
more insight about the advantages of our proposed CRLC-semi, we provide
detailed comparison between this method and the best SSL baseline
- FixMatch \cite{sohn2020fixmatch} in the next section.

\paragraph{Direct comparison between CRLC-semi and FixMatch\label{par:Direct-comparison-between-CRLC-n-FixMatch}}

FixMatch \cite{sohn2020fixmatch} is a powerful SSL method that makes
use of pseudo-labeling \cite{lee2013pseudo} and strong data augmentation
\cite{cubuk2019randaugment} to generate quality pseudo-labels for
training. FixMatch has been shown to work reasonably well with only
1 labeled sample per class. In our experiment, we observe that FixMatch
outperforms CRLC-semi on both CIFAR10 and CIFAR100. However, FixMatch
must be trained in much more steps than CRLC-semi to achieve good
results and its performance is very inconsistent (like other SSL baselines)
compared to that of CRLC-semi (Figs.~\ref{fig:Learning-curves-CRLC-FixMatch},
\ref{fig:CIFAR10_1L_3run_curves}).

\paragraph{Details of the labeled samples}

For the purpose of comparison and reproducing the results in Table~\ref{tab:Semi-supervised-learning-results-FULL},
we provide the indices of 40 labeled CIFAR10 samples and 400 labeled
CIFAR100 samples used in our experiments in Fig.~\ref{fig:Indices-40-labeled-CIFAR10}
and Fig.~\ref{fig:Indices-400-labeled-CIFAR100}, respectively. We
also visualize these samples in Fig.~\ref{fig:40-labeled-CIFAR10}
and \ref{fig:400-labeled-CIFAR100}. We note that we do not cherry-pick
these samples but randomly draw them from the training set.

\begin{figure*}
\begin{centering}
\includegraphics[width=0.235\textwidth]{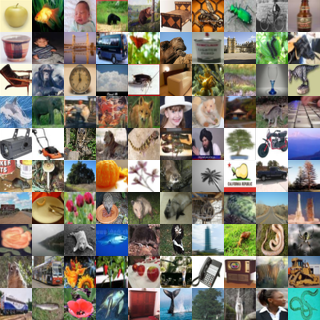}\hspace{0.015\textwidth}\includegraphics[width=0.235\textwidth]{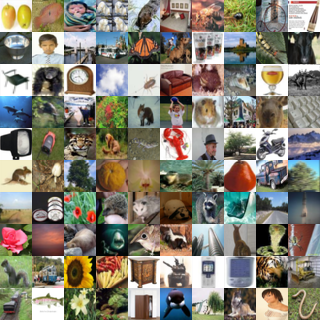}\hspace{0.015\textwidth}\includegraphics[width=0.235\textwidth]{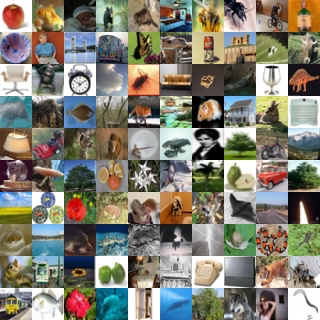}\hspace{0.015\textwidth}\includegraphics[width=0.235\textwidth]{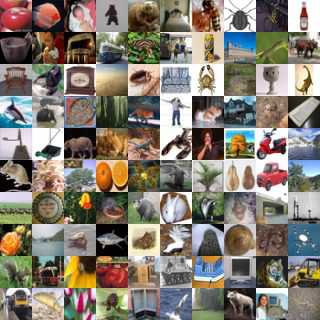}
\par\end{centering}
\caption{400 labeled CIFAR100 samples organized into 4 image blocks where each
image block is a set of 100 images corresponding to 100 classes. For
100 and 200 labeled samples, the first block and the first two blocks
are considered, respectively.\label{fig:400-labeled-CIFAR100}}
\end{figure*}

\begin{figure*}
\begin{centering}
\includegraphics[width=1\textwidth]{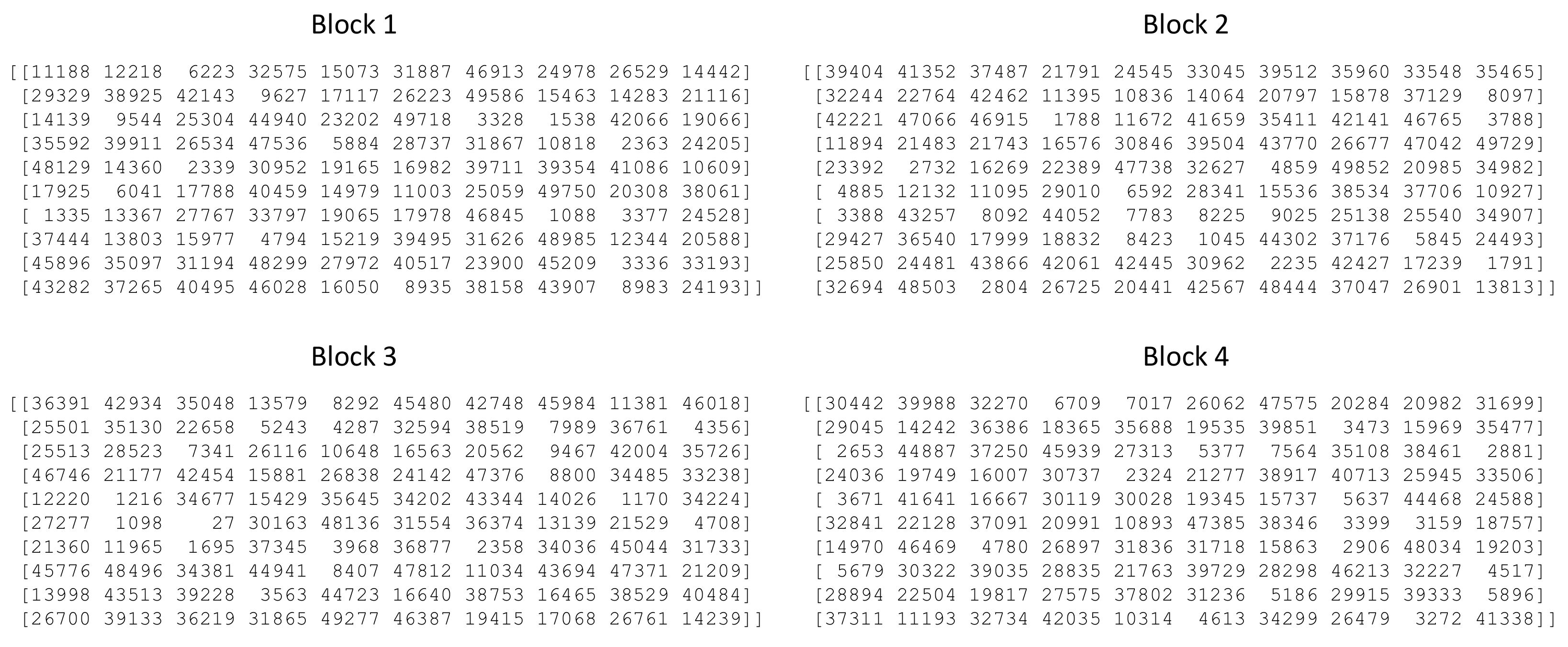}
\par\end{centering}
\caption{Indices in the training set of the images in Fig.~\ref{fig:400-labeled-CIFAR100}\label{fig:Indices-400-labeled-CIFAR100}}
\end{figure*}

\end{document}